\documentclass[twoside,11pt]{article}

\usepackage{jmlr2e}
\usepackage{amsmath,amsfonts}
\usepackage{bbold}
\usepackage{color}
\usepackage{enumitem}%
\usepackage{adjustbox}
\usepackage{booktabs}
\usepackage{multirow}

\usepackage{dcolumn}
\usepackage{booktabs}
\usepackage{tikz}
\usetikzlibrary{positioning,shapes,arrows}
\newcolumntype{M}[1]{D{.}{.}{1.#1}}

\usepackage{algorithm}
\usepackage{algpseudocode}
\renewcommand{\algorithmicrequire}{\textbf{Input:}}
\renewcommand{\algorithmicensure}{\textbf{Output:}}
\usepackage{xpatch}
\xpatchcmd{\algorithmic}{\setcounter}{\algorithmicfont\setcounter}{}{}
\providecommand{\algorithmicfont}{}

\renewcommand{\vec}[1]{\mathbf{#1}}

\def\S{{\vec{S}}}
\def\x{{\vec{x}}}
\def\X{{\vec{X}}}

\def\Y{{\vec{Y}}}
\def\z{{\vec{z}}}
\def\Z{{\vec{Z}}}

\def\E{{\mathbb E}}
\def\I{{\mathbb I}}
\def\II{{\mathbb 1}}
\def\P{{\mathbb P}}
\def\R{{\mathbb R}}
\def\Re{{\mathbb R}}
\def\V{{\mathbb V}}
\def\XX{{\mathbb X}}

\def\sF{{\mathcal F}}
\def\sL{{\mathcal L}}
\def\sO{{\mathcal O}}
\def\sS{{\mathcal S}}

\def\limn{\lim_{n \rightarrow \infty}}
\def\convas{\stackrel{a.s.}{\rightarrow}}
\def\convp{\stackrel{\P}{\rightarrow}}
\def\convd{\rightsquigarrow}

\newcommand{\quantmodel}{
 \begin{tikzpicture}[node distance=1cm and 0cm,
 mynode/.style={draw, ellipse, text width=1.5cm,
 align=center}]
  \node[mynode] (p) {$p_{1|L}$};
  \node[mynode,below right=of p] (l) {$(\XX_L,\Y_L)$};
  \node[mynode,above right=of l] (f) {$(f_0,f_1)$};
  \node[mynode,below right=of f] (u) {$\XX_U$};
  \node[mynode,above right=of u] (t) {$\theta$};
  \node[mynode,below right=of t] (uu) {$\Y_U$};
  \path (p) edge[-latex] (l)
  (f) edge[-latex] (l)
  (f) edge[-latex] (u)
  (t) edge[-latex] (uu)
  (uu) edge[-latex] (u);
 \end{tikzpicture}
}
\usepackage{amsmath}

\newtheorem{assumption}{Assumption}

\title{Quantification Under Prior Probability Shift: 
the Ratio Estimator and its Extensions}

\author{\name Afonso Fernandes Vaz 
\email afonsofvaz@gmail.com\\
 \name Rafael Izbicki 
 \email rafaelizbicki@gmail.com \\
 \name Rafael Bassi Stern 
 \email rbstern@gmail.com \\
 \addr Department of Statistics \\
 Federal University of S\~ao Carlos \\
 S\~ao Carlos, SP 13565-905, Brazil 
 }

\editor{Charles Elkan}

\usepackage{lastpage}
\jmlrheading{20}{2019}{1-\pageref{LastPage}}{7/18; Revised
2/19}{4/19}{18-456}{Afonso Fernandes Vaz, Rafael Izbicki and Rafael Bassi Stern}
\ShortHeadings{Quantification}{Vaz, Izbicki and Stern}
\firstpageno{1}

\begin{document}

\maketitle

\begin{abstract}%
 The quantification problem consists of
 determining the prevalence of a given label in
 a target population. However,
 one often has access to the labels in
 a sample from the training population but not in 
 the target population.
 A common assumption in this situation is
 that of prior probability shift, that is,
 once the labels are known,
 the distribution of the features is 
 the same in the training and target populations.
 In this paper, we derive a new lower bound
 for the risk of the quantification problem under
 the prior shift assumption.
 Complementing this lower bound,
 we present a new approximately minimax 
 class of estimators, ratio estimators, 
 which generalize several previous proposals
 in the literature. Using a weaker version of 
 the prior shift assumption, which can be tested,
 we show that ratio estimators can be used to 
 build confidence intervals for 
 the quantification problem.
 We also extend the ratio estimator so that it can:
 (i) incorporate labels from the target population,
 when they are available and
 (ii) estimate how the prevalence of
 positive labels varies according to
 a function of certain covariates.
\end{abstract}


\section{Introduction}
\label{sec::intro}

In several applications of binary classifiers,
predicting the labels of
individual observations \emph{per se} is
less important than evaluating 
the proportion of each label on 
an unlabeled target data set.
The latter task is called quantification
\citep{Forman2008}. For example,
a company may be interested in
evaluating the proportion of users who 
like each of their products, without access
to labeled reviews of these products.

A common approach to such a problem is to
(i) train a classifier for 
the user's evaluation based on
labeled reviews of other products, and
(ii) apply this classifier to the unlabeled 
target set and use the proportion of users who
are classified as liking the product 
as an estimator.
However, it is known that 
this two-step approach, 
known as ``classify and count", 
fails because of domain shift \citep{Forman2006,Tasche2016}.
In order to deal with this problem,  
several improvements have been proposed under an assumption named prior shift \citep{Saerens2002,Forman2008,Bella2010,Barranquero2015}.
A particular estimator that successfully
performs quantification is
the adjusted count (AC) estimator
\citep{Buck1966,Saerens2002,Forman2008}.
Part of the success of the AC estimator 
is explained in \citet{Tasche2017} by
showing that it is Fisher consistent.
However, there are more properties one
might desire of an estimator.

In order to investigate these properties,
\citet{Vaz2017} introduces the ratio estimator,
which is a generalization of the AC estimator.
\citet{Vaz2017} derives
the asymptotic mean squared error of the ratio estimator.
Here, we show that the ratio estimator is
approximately minimax and consistent under 
the prior probability shift assumption.
In order to derive this result, 
we prove a new lower bound for the risk of 
the quantification problem under
the prior shift assumption. 
This lower bound is general and applies to
every method under the 
prior probability shift assumption.
We also derive a central limit theorem for
the ratio estimator which helps to
explain its good performance and
leads to a method for building 
confidence intervals for 
the quantification problem.
This result also allows us to propose a
new type of ratio estimator based on
Reproducing Kernel Hilbert spaces.
Since the AC estimator and the method in
\citet{Bella2010} are special cases of
the ratio estimator, they benefit from
all of the results above.

It is important to
evaluate whether the prior probability shift
assumption indeed holds, otherwise the AC method
can perform poorly \citep{Tasche2017}.
We show that the ratio estimator
works under an assumption that is less
stringent than the prior shift assumption. 
Moreover, we show how this assumption can be tested.
We are not aware of other methods to 
test the prior shift and related assumptions.

We also generalize the ratio estimator to
two extensions of the quantification problem.
In the first scenario, some labels are
available in the target population.
The combined estimator extends 
the ratio estimator in order to 
incorporate these labels and obtain 
a larger effective sample size.
The second scenario considers that 
the prevalence of each label varies
according to additional covariates.
This generalization allows one to 
use unlabeled data to identify e.g. how
the approval of a product varies with age.
In this scenario, we introduce the
regression ratio estimator,
which offers improvements over
the standard methods that are used in
sentiment analyses \citep{Wang2012}.

Section \ref{sec::quantification} discusses
the standard quantification problem under
the prior probability shift assumption.
Subsection \ref{sec:lower-bound} provides new
lower bounds for the risk in this scenario.
Subsection \ref{sec::ratioEstimator} introduces
the ratio estimator, uses the result from
the previous subsection to show that
it is approximately minimax and also derives
its convergence rate and a central limit theorem.
Subsection \ref{sec:rkhs} uses the
asymptotic behavior of the ratio estimator to
propose a new type of ratio estimator based on
Reproducing Kernel Hilbert spaces.
Finally, the ratio estimator requires 
a weaker version of prior probability shift to
obtain consistency.
Subsection \ref{sec:hypothesis} discusses
a new algorithm for testing this assumption.

Section \ref{sec:extensions} proposes extensions of
the ratio estimator to scenarios which 
are more general than 
the standard quantification problem.
Subsection \ref{sec:combined} proposes
the combined estimator, for cases in which
some labels are available in 
the population of interest.
Subsection \ref{sec::regression} proposes the
ratio regression estimator, for the situation
in which the prevalence of a given label varies
according to a covariate.
All proofs are presented in the appendix;
code and data used for the experiments
is available at \url{https://github.com/afonsofvaz/ratio_estimator}.

\section{Quantification under 
prior probability shift}
\label{sec::quantification}

In order to formally approach
the quantification problem,
we use the same notation as in \citet{Wasserman2006}.
If $\Z_1 \in \Re^{d_1}$ and $\Z_2 \in \Re^{d_2}$ are
random vectors and $R \subset \Re^{d_1}$, then
$\P(\Z_1 \in R|\Z_2)$ is the conditional probability that
$\Z_1$ is in $R$ given $\Z_2$. Using $\P$,
one can obtain $F_{\Z_1|\Z_2}$, $f_{\Z_1|\Z_2}$,
$\E[\Z_1|\Z_2]$, and $\V[\Z_1|\Z_2]$ which are, respectively,
the conditional distribution, density, expected value
and variance of $\Z_1$ given $\Z_2$.
Marginal properties of $\Z_1$ are indicated by
omitting the conditioning random variable.
Also, if $(\Z_n)_{n \in \mathbb{N}}$ is
a sequence of random vectors, then
$\Z_n \convas \Z$, $\Z_n \convp \Z$, and
$\Z_n \convd \Z$ indicate respectively, that
$\Z_n$ converges almost surely, in probability, 
and in distribution to $\Z$. In order to
express the rate at which convergence occurs,
it is useful to use $\sO$ and $\Omega$ notation.
If $(a_n)_{n \in \mathbb{N}}$ is
a sequence in $\Re$, then
$a_n = \sO(g(n))$ if there exists $c$ such that,
for every $n$, $a_n \leq c \cdot g(n)$ and
$a_n = \Omega(g(n))$ if there exists $c$ such that,
for every $n$, $a_n \geq c \cdot g(n)$.
Finally, $\I$ is the indicator function.
An expression such as $\I(g(\X) \in A)$ is
equal to $1$ when $g(\X) \in A$ and
to $0$ when $g(\X) \notin A$.

In the quantification problem,
for each sample instance $i \in \{1,\ldots,n\}$,
$(\X_i,Y_i,S_i)$ is a vector of
random variables such that
$\X_i \in \R^d$ are features, 
$Y_i \in \{0,1\}$ is a label of interest and
$S_i \in \{0,1\}$ is the indicator that 
this instance has been labeled.
That is, whenever $S_i=0$, then
$Y_i$ is not observed. Note that
$S_i$ can be random.

In the above framework,
some subsets of the instances
are frequently used. The sets
$A_k := \{i\in \{1,\ldots,n\}: S_i=k\}$ represent
the labeled ($k=1)$ and unlabeled ($k=0$) instances.
Similarly, $A_{k,j}
:= \{i \in \{1,\ldots,n\}: S_i=k\mbox{ and } Y_i=j\}$ 
represent the instances that are labeled ($k=1$) 
or unlabeled ($k=0$) and have a
positive ($j=1$) or a zero ($j=0$) label.
Also the number of instances that are unlabeled,
labeled or that have label $j$ are denoted,
respectively, by $n_{U} := |A_0|$,
$n_{L} := |A_1|$ and $n_j := |A_{1,j}|$.

In a quantification problem, one wishes
to estimate $\theta := \P(Y=1|S=0)$,
that is, the prevalence of 
positive labels among unlabeled samples.
This prevalence is not assumed to be the same
as the one over labeled sets, $\P(Y=1|S=1)$.
The estimator for $\theta$ can 
depend only on the available data,
that is, the features of all instances and 
the labels that were obtained.
Formally, letting $Z_i^j$ denote $(Z_i,\ldots,Z_j)$,
a valid estimator is a function of
$\X_1^n$, $S_1^n$ and
$(Y_{i})_{i \in A_1}$. The set of
all such valid estimators is denoted by $\sS$.

In the standard formulation of 
the prior probability shift problem, 
$\{(\X_i,Y_i)\}_{i \in A_0}$ is called 
the \emph{target population} 
(since the labels are unavailable), and
$\{(\X_i,Y_i)\}_{i \in A_1}$
is called the training population 
\citep{Tasche2017}.
It is common for both populations to 
be i.i.d.,
\begin{assumption} \
 \label{assump::iidRelaxed}
 \begin{itemize}
  \item  $(S_1,\X_1,Y_1),\ldots,(S_n,\X_n,Y_n)$ 
  are independent.
  \item For every $s \in \{0,1\}$, 
  $(\X_1,Y_1)|S_1=s,\ldots,(\X_n,Y_n)|S_n=s$
  are identically distributed.
 \end{itemize}
\end{assumption}

Unless additional assumptions are made,
it is not possible to 
learn about $\theta$ using solely 
the observed data.
One assumption that allows
learning about $\theta$ is the
prior probability shift,
which states that
``the class-conditional feature distributions of 
the training and test sets are the same" \citep{Fawcett2005}. 
Prior shift is formalized in 
Assumption \ref{ass::priorShift}.
\begin{assumption}
 \label{ass::priorShift}
 [\textbf{Prior probability shift}]
 \label{assump::priorShift}
 For every $(y_1,\ldots,y_n) \in \{0,1\}^{n}$,
 $(\X_1,\ldots,\X_n)$ is stochastically independent of 
 $(S_1,\ldots,S_n)$ 
 conditionally on 
 $(Y_1,\ldots,Y_n) = (y_1,\ldots,y_n)$.
\end{assumption}

Although Assumption \ref{ass::priorShift} is
written in a different way than in
papers such as \citet{Moreno2012},
the content is similar. While
\citet{Moreno2012} uses a subscript on
the probability function to determine which
is the reference population,
we perform this task using the random variable, $S$.
For instance, the probability that
an instance from the target population has
the label ``1'' is referred in 
previous notation and in this paper,
respectively, as
$\P_{tg}(Y_i = 1)$ and $\P(Y_i = 1|S_i = 0)$.
Using this translation, Assumption \ref{ass::priorShift}
is the same as the prior probability shift in
\citet{Moreno2012}. Assumption \ref{ass::priorShift} holds
if and only if $f_{\X|Y,S=0} \equiv f_{\X|Y,S=1}$,
that is, $\P_{tg}(\x|y) \equiv \P_{tr}(\x|y)$.

\subsection{Lower bound on the risk for
quantification under prior probability shift}
\label{sec:lower-bound}

Under Assumptions \ref{assump::iidRelaxed} and
\ref{ass::priorShift} it is possible to learn about
$\theta$ from the features and labels that
are available in the quantification problem.
For example, one can use the features and labels in
the training population to learn about
$f_{\X|Y=0}$ and $f_{\X|Y=1}$.
Also, if these densities are sufficiently different, 
then one can combine the information about them
to the features in the target population to
learn about the unknown labels in this population and,
therefore, about $\theta$.
Definition \ref{defn:separability} formally presents
two classes in which the possible values of
$f_{\X|Y=0}$ and $f_{\X|Y=1}$ are separable.

\begin{definition}
 \label{defn:separability}
 Let $f_i(\x) = f_{\X|Y=i}(\x)$,
 $\epsilon, K > 0$ and
 $g: \Re^d \rightarrow \Re$ be
 a non-constant function.
 \begin{align*}
  \begin{cases}
   \sF_{\sL_1,\epsilon} := \left\{(f_0,f_1): 
   \|f_0-f_1\|_1 \geq \epsilon\right\} \\
   \sF_{g,K,\epsilon} := 
   \left\{\left(f_0,f_1\right):
   \E_{f_i}[g(\X)^2|Y=i] \leq K,
   \text{ and } 
   |\E_{f_1}[g(\X)|Y=1]-\E_{f_0}[g(\X)|Y=0]|
   \geq \epsilon \right\}
  \end{cases}
 \end{align*}
\end{definition}

Under the classes in 
Definition \ref{defn:separability} 
it is possible to learn about $\theta$
and the learning rate
depends on both 
the number of labeled and 
unlabeled instances.
A lower bound for how these sample sizes
affect the rate at which one learns about $\theta$
is presented in Theorem
\ref{thm:lower_bound}.

\begin{definition}
 Let $\sF$ be a collection of $(f_0,f_1)$.
 The minimax rate, $M(\sF)$, 
 for estimating $\theta$
 under the squared loss, $\mathcal{F}$, and 
 Assumptions \ref{assump::iidRelaxed} and
 \ref{ass::priorShift} is 
 \begin{align*}
  M(\mathcal{F}) &= 
  \inf_{\widehat{\theta} \in \sS}
  \sup_{(f_0,f_1) \in \mathcal{F}; \theta \in [0,1]}
  \E_{f_0,f_1,\theta}\left[(\widehat{\theta}-\theta)^2
  \bigg|S_1^n\right] 
 \end{align*}
\end{definition}

\begin{theorem}
 \label{thm:lower_bound}
 $M(\sF_{\sL_1, \epsilon}) \geq
 \Omega(\max(n_L^{-1},n_U^{-1}))$ and
 $M(\sF_{g,K,\epsilon}) \geq
 \Omega(\max(n_L^{-1},n_U^{-1}))$.
\end{theorem}

\autoref{thm:lower_bound}
shows that it is not possible
to obtain an estimator
for $\theta$ which has convergence
rate faster than $\Omega(\max(n_L^{-1},n_U^{-1}))$.
In particular, it is not possible to learn $\theta$
by observing solely a limited amount of labels.
The following subsection introduces
the ratio estimator for $\theta$, which
achieves the lower bound in
\autoref{thm:lower_bound} under
$\sF_{g,K,\epsilon}$.

\subsection{The ratio estimator and 
its theoretical properties}
\label{sec::ratioEstimator}
 
\begin{definition}[Ratio estimator]
 \label{defn:ratio}
 Let $g: \R^d \longrightarrow \R$.
 The untrimmed ratio estimator for $\theta$ 
 based on $g$,
 $\widehat{\theta}_{UR}$, is
 \begin{align*}
  \widehat{\theta}_{UR}
  &:= \frac{\frac{\sum_{i \in A_0} g(\X_i)}{n_U} 
  - \frac{\sum_{i\in A_{1,0}}g(\X_i)}{n_0}}
  {\frac{\sum_{i\in A_{1,1}}g(\X_i)}{n_1} 
  - \frac{\sum_{i\in A_{1,0}}g(\X_i)}{n_0}}
 \end{align*}
 Since $\theta \in [0,1]$,
 the ratio estimator,
 $\widehat{\theta}_{R}$, is
 \begin{align*}
  \widehat{\theta}_{R}
  &= \max(0,\min(1,\widehat{\theta}_R))
 \end{align*}
\end{definition}

The ratio estimator generalizes estimators which
were previously proposed in the literature.
This fact follows from observing that
the terms in the untrimmed ratio estimator are
sample averages of $g(\X)$ among three groups of instances:
unlabeled instances, instances labeled as $0$, and
instances labeled as $1$.
For instance, the adjusted count (AC) estimator
\citep{Buck1966,Saerens2002,Forman2008,Tasche2017} is
the a ratio estimator when $g(\x) \in \{0,1\}$, that is, 
$g(\x)$ is the output of a classifier for $Y$.
Also, the estimator in \citet{Bella2010} is
a ratio estimator when $g(\x)=\widehat{\P}(Y=1|\x)$,
that is, $g(\x)$ is a soft classifier for $Y$.

\begin{remark}
 The ratio estimator can 
 be generalized to the case in which
 $Y_i \in \{0,1,\ldots,k\}$.
 In this case, let 
 $g: \R^d \rightarrow \R^k$ be
 a fixed function.
 By defining $G$ as a $k \times (k+1)$ matrix
 such that $G_{i,j}=\E[g_i(\X)|Y=j-1,S=1]$,
 $p \in \R^{k+1}$ such that 
 $p_i = \P(Y=j-1|S=0)$, and
 $g \in \R^{k}$ such that 
 $g_i = \E[g_i(\X)|S=0]$, 
 $\widehat{\theta}_{UR}$ is obtained by
 solving the linear system
 \begin{align*}
  \begin{cases}
   \widehat{g} 
   &= \widehat{G} \cdot \widehat{\theta}_{UR} \\
   1 &= \textbf{1}^{t} \cdot \widehat{\theta}_{UR}
  \end{cases} 
  & \text{, where }
  \widehat{g}_i 
  = \frac{\sum_{k \in A_0} g_i(\X_k)}{n_U}
  \text{ and }
  \widehat{G}_{i,j} 
  = \frac{\sum_{k \in A_{1,j}}g_i(\X_k)}{n_{j}}
 \end{align*}
 Since $\widehat{\theta}_{UR}$ might have
 negative components, it is
 generally inadmissible according to
 the squared error \citep{Finetti2017}[p.90-91] that is,
 there exist estimators which have a squared error
 strictly smaller than $\widehat{\theta}_{UR}$.
 The ratio estimator, $\widehat{\theta}_{R}$ satisfies
 this property and is the projection of
 $\widehat{\theta}_{UR}$ onto 
 the simplex \citep{Michelot1986}:
 $\widehat{\theta}_{R} 
 = \arg\min_{\hat{p}:
 \hat{p} \geq 0, \sum_{i}\hat{p}_i = 1}
 \|\widehat{\theta}_{UR}-\hat{p}\|_2^2$.
\end{remark}

Similarly to the AC estimator \citep{Tasche2017},
the ratio estimator is Fisher consistent 
under weak assumptions.\footnote{
Although Fisher consistency is typically not equivalent to 
consistency in probability \citep{Gerow1989,Kass2011}, in the sequence we show that the ratio estimator is consistent in both senses.}. 
They are described in
Assumptions \ref{assump::cool1} and
\ref{assump::separability}.

\begin{assumption}[Weak prior shift]
 \label{assump::cool1}
 The function, $g$, is such that
 $g(\X)_{1}^{n}$ is stochastically independent of 
 $\S_{1}^{n}$ conditionally on 
 $\Y_{1}^{n}=y_{1}^n$.
\end{assumption}

\begin{assumption}[Separability]
 \label{assump::separability}
 The function, $g$, is such that
 \begin{enumerate}
  \item $\E[g(\X_i)|Y_i=j,S_i=1]$ are defined,
  for $j \in \{0,1\}$. 
  \item $\E[g(\X_i)|Y_i=1,S_i=1]
  -\E[g(\X_i)|Y_i=0,S_i=1] \neq 0$
 \end{enumerate}
\end{assumption}

%Assumptions \ref{assump::cool1} and
%\ref{assump::separability} require
%3 conditions of $g(\x)$.
The condition in Assumption \ref{assump::cool1}
is a relaxed type of 
prior probability shift that
is strictly weaker than 
Assumption \ref{ass::priorShift}.
Assumption \ref{assump::separability} requires
two more conditions of $g(\x)$.
According to condition 1, 
the population versions of
the expectations in Definition \ref{defn:ratio}
are defined.
Condition 2 states that
the ratio estimator calculated on
these population parameters is defined, that is,
there is no division by $0$.
\begin{theorem}
 \label{thm:fisher-ratio}
 Under Assumptions \ref{assump::iidRelaxed},
 \ref{assump::cool1} and 
 \ref{assump::separability},
 $\widehat{\theta}_{UR}$ and
 $\widehat{\theta}_{R}$ are
 Fisher consistent for $\theta$.
\end{theorem}

It is also possible to
guarantee a finite population bound on
the mean squared error of 
$\widehat{\theta}_{R}$.
This result is obtained in \autoref{thm:mse-trimmed},
which substitutes Assumption \ref{assump::separability}
by the stronger condition that
$(f_0,f_1) \in \sF_{g,K,\epsilon}$.

\begin{theorem}
 \label{thm:mse-trimmed}
 Under Assumptions \ref{assump::iidRelaxed} and
 \ref{assump::cool1},
 \begin{align*}
  \sup_{(f_0,f_1) \in \sF_{g,K,\epsilon}}
  \E_{f_0,f_1}\left[\left(\widehat{\theta}_{R}
  -\theta\right)^2\bigg|S_{1}^{n}\right] 
  \leq \sO(\max(n_{L}^{-1},n_{U}^{-1}))
 \end{align*}
\end{theorem}

Under the assumptions of \autoref{thm:mse-trimmed},
if $n_U \gg n_{L}$, then
the convergence of the
mean squared error of 
the ratio estimator is 
the same as the one that 
would have been obtained if 
one observed solely $n_{L}$ labels from
the target population and used 
the sample's label proportions to
estimate $\theta$.
The same type of result cannot generally 
be obtained for the untrimmed ratio estimator,
since the trimming is necessary to
guarantee that the ratio of random variables
does not have infinite variance.
While these conclusions are similar to
the ones obtained from Theorem 3 in \citet{Lipton2018},
there exist two main differences.
First, while the former assumes that there are 2 labels only,
the latter applies to an arbitrary number of labels.
Second, \autoref{thm:mse-trimmed} upper bounds
the squared error by $\sO(\max(n_{L}^{-1},n_{U}^{-1}))$,
which is slightly tighter than the bound of
$\sO\left(\max\left(\frac{\log n_L}{n_{L}},
\frac{\log n_U}{n_{U}}\right)\right)$
in \citet{Lipton2018}.

It follows from \autoref{thm:lower_bound} and
\autoref{thm:mse-trimmed} that the ratio estimator satisfies
several desirable properties. These properties are
presented in Definition \ref{defn:approx_minimax} and
Corollary \ref{cor:consistency}.

\begin{definition}
 \label{defn:approx_minimax}
 Let $\sS$ and $\sF$ be, respectively,
 the classes of estimators and
 distributions over the data under consideration.
 An estimator $\widehat{\theta}^* \in \sS$ is
 approximately minimax for
 estimating $\theta$ under
 the squared error loss if
 \begin{align*}
  \sO\left(\sup_{(f_0,f_1) \in \sF_{g,K,\epsilon}}
  \E_{f_0,f_1}\left[\left(\widehat{\theta}^*
  -\theta\right)^2\bigg|S_{1}^{n}\right]\right)
  &= \Omega\left(\inf_{\widehat{\theta} \in \sS}
  \sup_{(f_0,f_1) \in \mathcal{F}; \theta \in [0,1]}
  \E_{f_0,f_1,\theta}\left[(\widehat{\theta}-\theta)^2
  \bigg|S_1^n\right]\right)
 \end{align*}
 That is, the squared error of $\widehat{\theta}^*$
 attains the optimal rate of convergence.
\end{definition}

\begin{corollary}
\label{cor:consistency}
Under Assumptions \ref{assump::iidRelaxed} and
 \ref{assump::cool1}, if there exists
 $\epsilon, K > 0$ such that 
 $(f_0,f_1) \in \sF_{g,K,\epsilon}$, then
 $\widehat{\theta}_{R}$ is 
 consistent for $\theta$ in probability and
 in $\sL_2$ as
 $n_U \convp \infty$ and
 $n_L \convp \infty$.
 Also, under Assumptions \ref{assump::iidRelaxed},
 \ref{assump::cool1}, and $\sF_{g,K,\epsilon}$, 
 $\widehat{\theta}_{R}$ is approximately minimax.
\end{corollary}

Corollary \ref{cor:consistency} shows that
the ratio estimator converges to $\theta$
under a weaker version of the 
prior probability shift assumption and that 
the rate of this convergence is minimax (i.e., it is the  same rate as that of the minimax estimator).
Since the estimators from
\citet{Buck1966,Saerens2002,Forman2008,Bella2010} are
particular cases of the untrimmed ratio estimator,
their trimmed versions also converge to $\theta$
under the weak prior shift.

The ratio estimator also
satisfies a central limit theorem.
In order to obtain this result,
besides requiring Assumptions 
\ref{assump::iidRelaxed},
\ref{assump::cool1} and
\ref{assump::separability},
it is also necessary to require that
conditionally on $Y$,
$g(\X)$ has bounded variance and that
the number of labeled samples goes to infinity.
These conditions are described in
Assumption \ref{assump::ratio-clt}.
The central limit theorem is presented in
\autoref{thm:ratio-clt}.

\begin{assumption} \
 \label{assump::ratio-clt}
 \begin{enumerate}
 \item $\V[g(\X_i)|Y_i=j] < \infty$, 
  for every $j \in \{0,1\}$.
  \item There exists $h(n) \geq 0$ such that
  $\limn \frac{h(n)}{n} < 1$, 
  $\limn h(n) = \infty$, and
  $\frac{n_L}{h(n)} \convp 1$.
 \end{enumerate}
\end{assumption}

\begin{theorem}
 \label{thm:ratio-clt}
 \label{thm::EQMPriorShift}
 Define $\mu_{j} := \E[g(\X_1)|Y_1=j]$, 
 $\sigma_{j}^2 := \V[g(\X_1)|Y_1=j]$,
 $p_{L} := \limn \frac{h(n)}{n}$, and
 $p_{j|L} := \P(Y=j|S=1)$.
 Under Assumptions \ref{assump::iidRelaxed},
 \ref{assump::cool1},
 \ref{assump::separability} and
 \ref{assump::ratio-clt},
 \begin{enumerate}
  \item If $p_L \neq 0$, then
  \begin{align*}
   \sqrt{n}(\widehat{\theta}_{R}-\theta)
   &\convd N\left(0,
   \frac{\frac{(1-\theta)\sigma_0^2 + \theta \sigma_1^2 + (\mu_1-\mu_0)^2 \theta(1-\theta)}{1-p_L}
   +\frac{(1-\theta)^2 \sigma_0^2}{p_L p_{0|L}}
   +\frac{\theta^2 \sigma_1^2}{p_L p_{1|L}}}
   {(\mu_1-\mu_0)^2} \right)
  \end{align*}
  \item If $p_L = 0$, then
  \begin{align*}
   \sqrt{h(n)}(\widehat{\theta}_{R}-\theta)
   &\convd N\left(0,
   \frac{\frac{(1-\theta)^2 \sigma_0^2}{p_{0|L}}
   +\frac{\theta^2 \sigma_1^2}{p_{1|L}}}
   {(\mu_1-\mu_0)^2} \right)
  \end{align*}
 \end{enumerate}
\end{theorem}

It is possible to use \autoref{thm::EQMPriorShift}
to obtain an 
approximate confidence interval for $\theta$.
This interval is obtained by
inverting the convergence results in
\autoref{thm::EQMPriorShift}, and
substituting $\theta$ for
$\widehat{\theta}_{R}$ and
the population parameters,
$\mu_0$, $\mu_1$, $\sigma_{0}^2$,
$\sigma_1^2$, $p_L$, $p_{0|L}$ and $p_{1|L}$, by 
their respective empirical averages.
This confidence interval
may also be used to 
test hypothesis such as 
$H_0: \theta \in \Theta_0$.

\autoref{thm::EQMPriorShift} also
provides an approximation for
the mean squared error of 
$\widehat{\theta}_{R}$.
This approximation for the common case
in which $n_U \gg n_{L}$ is presented
in the following corollary.

\begin{corollary}
 \label{cor::MSE}
 Under Assumptions
 \ref{assump::iidRelaxed}, 
 \ref{assump::cool1},
 \ref{assump::separability} and 
 \ref{assump::ratio-clt},
 if $p_{L} = 0$ $(n_U \gg n_{L})$, then
 \begin{align}
  \label{eq::MSELargeNL}
  \mbox{MSE}(\widehat{\theta}_{R}) 
  &\approx \frac{1}{n_L (\mu_1-\mu_0)^2 }
  \left(\frac{\sigma_0^2 (1-\theta)^2}{p_{0|L}}
  +\frac{\sigma_1^2 \theta^2}{p_{1|L}} \right)
 \end{align}
\end{corollary}

Corollary \ref{cor::MSE}
brings some insights on how
$g$ should be chosen in order for
$\widehat{\theta}_{R}$ to be
an accurate estimator of $\theta$.
For instance,
it shows that one should choose $g$ such that 
$|\mu_1-\mu_0|$ is large and 
both $\sigma_0^2$ and $\sigma_1^2$ are small. 
This implies that the distributions of 
$g(\X)|Y= 1$ and $g(\X)|Y= 0$ should 
place most of their masses in regions that
are far apart.
This conclusion explains the success
of the methods in 
\citet{Forman2008}, in which $g(\x)$ is 
a classifier, and 
\citet{Bella2010}, in which $g(\x)$ is 
an estimate of $\P(Y=1|\x)$.

One of the main deficiencies of the standard AC estimator
is that its denominator can be very close to zero, which makes it very unstable (due to a large variance). In order to handle this, we can explicitly use the approximation of the MSE 
(Corollary \ref{cor::MSE})
to choose better functions  $g$.
This procedure is discussed in
the following subsection.

\subsection{Choosing \textbf{g} via 
approximate MSE minimization}
\label{sec:rkhs}

One possible criterion for the choice of $g$ is
the minimization of $MSE(\widehat{\theta}_{R})$,
defined in Corollary \ref{cor::MSE}.
However, the latter depends on 
unobservable quantities.
An alternative is to minimize an estimate of
$MSE(\widehat{\theta}_{R})$.
This estimate is presented in
Definition \ref{def:empirical-mse}.

\begin{definition}
 \label{def:empirical-mse}
 Let $\hat{\theta}$ be an estimator of $\theta$ and,
 for each $i \in \{0,1\}$, let
 \begin{align*}
  \widehat{\mu}_i
  &= n_i^{-1}\sum_{A_{1,i}}{g(\X_i)} 
  & \widehat{\sigma}_{i}^2
  &= n_i^{-1}\sum_{A_{1,i}}{(g(\X_i)-\widehat{\mu}_i)^2}
  & \widehat{p}_{i|L}
  = \frac{n_i}{n_0+n_1}
 \end{align*}
 The empirical MSE of 
 the ratio estimator induced by $g$,
 $\widehat{MSE}(g)$, is 
 \begin{align*}
  \widehat{\mbox{MSE}}(g)
  &\approx \frac{1}{n_L(\widehat{\mu}_1-\widehat{\mu}_0)^2}
  \left(\frac{\widehat{\sigma}_0^2(1-\widehat{\theta})^2}
  {\widehat{p}_{0|L}}
  +\frac{\widehat{\sigma}_1^2 \widehat{\theta}^2}
  {\widehat{p}_{1|L}} \right)
 \end{align*}
\end{definition}

In order to avoid overfitting,
we perform the minimization of
$\widehat{MSE}(g)$ on 
a  Reproducing Kernel Hilbert Space 
(RKHS; \citet{Wahba1990}).
More precisely, if 
$K$ is a Mercer kernel and
$\mathcal{H}_K$ is the RKHS
associated to $K$, then
we choose $g^*$ as 
\begin{align}
 \label{eq:rkhs}
 g^* 
 &:= \arg\min_{g \in \mathcal{H}_K}
 \widehat{\mbox{MSE}}(g)
\end{align}

In the following,
\autoref{thm:rkhs} presents
a characterization of $g^*$ in 
eq. \ref{eq:rkhs}.

\begin{theorem}
 \label{thm:rkhs}
 Let $K$ be a Mercer kernel and 
 $\mathcal{H}_K$ the corresponding RKHS. Also,
 \begin{itemize}
  \item $\mathbb{K}$: the Gram matrix defined
  for $(i,j) \in A_{1}^2$ and such that
  $(\mathbb{K})_{i,j}= K(\x_i,\x_j)$.
  \item $m_i$: A vector of size $|A_1|$ and
  such that, for each $k \in A_1$, 
  $m_{i,k} = \frac{\sum_{j \in A_{1,i}} K(\x_j,\x_k)}{n_i}$.
  \item $M=(m_1-m_0)(m_1-m_0)^t$.
  \item $\widehat{\Sigma}_i$: 
  a $|A_1| \times |A_1|$ matrix such that
  $(\widehat{\Sigma}_i)_{k,l}$ is
  the sample covariance between
  $(K(\x_j,\x_k))_{j \in A_{1,i}}$ and
  $(K(\x_j,\x_l))_{j \in A_{1,i}}$.
  \item $N$:
  a $|A_1| \times |A_1|$ matrix such that
  $N=\frac{\widehat{\theta}^2}{\widehat{p}_{1|L}}\widehat{\Sigma}_1+\frac{(1-\widehat{\theta})^2}{\widehat{p}_{0|L}}\widehat{\Sigma}_0$.
  \item $\vec{w}^* = \arg\min_{w \in \Re^{n_L}} 
  \frac{\vec{w}^t N \vec{w}}
  {\vec{w}^t M \vec{w}}$
 \end{itemize}
 The function $g^*$ in eq. \ref{eq:rkhs} satisfies
 $g^*(\x) = \sum_{i \in A_1} w_i^* K(\x,\x_i)$.
\end{theorem}

The vector, $\vec{w}^*$ in \autoref{thm:rkhs} is
the eigenvector associated to
the largest eigenvalue in absolute value, $\lambda^*$, 
of the generalized eigenvalue problem,
$M \vec{w}^* =\lambda^* N \vec{w}^*$.
If $N$ is invertible, 
$\vec{w}^*$ is the eigenvector associated to
the largest eigenvalue in 
absolute value of $N^{-1}M$.
Alternatively, if $N$ is not invertible
one can substitute $N$ in 
\autoref{def:empirical-mse} by
$(N+\gamma \II)^{-1}$, where 
$\II$ is the identity matrix and 
$\gamma$ is a small number that 
makes $N+\gamma \II$ invertible.
Adding $\gamma$ to the diagonal also
also adds regularization and can
therefore lead to an improved solution.
In practice we choose $\gamma$ via data-splitting.

The results in this and in the previous section 
rely on the weak prior shift assumption.
As shown in the next subsection,
one of the advantages of this assumption to
the regular prior shift is that 
it is easier to test.

\subsection{Testing the weak prior shift assumption}
\label{sec:hypothesis}

The following proposition is useful for
testing the weak prior shift assumption:

\begin{proposition}
 \label{lemma:priorShiftConsequence}
 Under Assumption \ref{assump::cool1},
 there exists $0\leq p\leq  1\mbox{ such that }$ 
 \begin{align*}
  pF_{g(\X)|S=1,Y=1} + (1-p)F_{g(\X)|S=1,Y=0} 
  &= F_{g(\X)|S=0}.
 \end{align*}
\end{proposition}

It follows from Proposition
\ref{lemma:priorShiftConsequence} that
Assumption \ref{assump::cool1} entails
the hypothesis:
\begin{align*}
 H_0:\exists 0\leq p\leq  1\mbox{ such that } pF_{g(\X)|S=1,Y=1} + (1-p)F_{g(\X)|S=1,Y=0} = F_{g(\X)|S=0}
\end{align*}

In the following, we show how to
construct an hypothesis test for $H_0$ when
$g(\x)$ is continuous.
Since the weak prior shift entails $H_0$,
if this test is used to test the weak prior shift,
it will have the correct type I error.
However, $H_0$ may hold when 
Assumption \ref{assump::cool1} is false.
%\red{add:}
% One can observe this fact intuitively by considering that
% $F_{g(\X)|S=0}$, $F_{g(\X)|S=0,Y=0}$ and $F_{g(\X)|S=0,Y=1}$
% are given and $F_{g(\X)|S=1,Y=0}$ and that $F_{g(\X)|S=1,Y=1}$
% are variables. While $H_0$ has a single equation for two variables,
% $F_{g(\X)|S=0} = (1-p)F_{g(\X)|S=1,Y=0} + pF_{g(\X)|S=1,Y=1}$,
% Assumption \ref{assump::cool1} specifies
% two equations for the same variables:
% $F_{g(\X)|S=0,Y=i} = F_{g(\X)|S=1,Y=i}$, for $i \in \{0,1\}$.
% Therefore, one might imagine that
% Assumption \ref{assump::cool1} specifies a space that
% has a lower dimension than $H_0$.
 A specific example in which $H_0$ is satisfied but
 Assumption \ref{assump::cool1} is not satisfied is
 given in the following example.
 \begin{example}
  \label{ex:minima}
  If $F_{g(\X)|S=0,Y=1}=F_{g(\X)|S=0,Y=0}$
and $F_{g(\X)|S=0,Y=1}=F_{g(\X)|S=1,Y=1}$,
then
 there exists $p$ such that $pF_{g(\X)|S=1,Y=1} + (1-p)F_{g(\X)|S=1,Y=0} = F_{g(\X)|S=0}$ (namely, $p=1$)
 even if $F_{g(\X)|S=1,Y=0} \neq F_{g(\X)|S=0,Y=0}$.
 In this case, Assumption \ref{assump::cool1} does not
 hold and $H_0$ is satisfied.
 \end{example}

The following statistic, $T$, measures
disagreement with $H_0$:
\begin{align*}
 T &= \inf_{0\leq p\leq 1}
 d\left(p\widehat{F}_{g(\X)|S=1,Y=1}+ 
 (1-p)\widehat{F}_{g(\X)|S=1,Y=0},
 \widehat{F}_{g(\X)|S=0}\right),
\end{align*}
where $d$ is a distance between
cumulative distributions, such as
the Kolmogorov distance, and
$\widehat{F}$ are 
the empirical cumulative distributions \citep{Wasserman2013}:
\begin{align*}
 \widehat{F}_{g(\X)|S=1,Y=i}(w)
 &=\frac{1}{|A_{1,i}|}
 \sum_{i \in A_{1,i}}\I(g(\X_i) \leq w),
 & \widehat{F}_{g(\X)|S=0}(w)
 &=\frac{1}{|A_{0}|}
 \sum_{i \in A_{0}}\I(g(\X_i) \leq w).
\end{align*}
Algorithm \ref{alg:MC}, which is presented below,
obtains a p-value for $H_0$ based on $T$.
The algorithm uses kernel smoothers
\citep{Wasserman2013} to estimate 
the conditional densities of $g(\X)$ given $Y$,
$f(g(\x)|Y=0)$ and $f(g(\x)|Y=1)$,
by $\widehat{f}(g(\x)|Y=0)$ and 
$\widehat{f}(g(\x)|Y=1)$.

\begin{algorithm}
 \caption{\small p-value for testing the 
 weak prior shift assumption}
 \label{alg:MC}
 \algorithmicrequire \ {\small Labeled and unlabeled sample, number of Monte Carlo simulations, $B$} %Due to the orthogonality of the spectral basis,
  %\Comment{$\epsilon$'s are the tuning parameters associated with the kernel}  
 \algorithmicensure \ {\small p-value}
 \begin{algorithmic}[1]
  \State  Compute $T_{\mbox{obs}}$,
  the test statistic for the observed data
  \State Compute $\widehat{\theta}$, an
  estimate of $\theta$
  \State Compute $\widehat{f}(g(\x)|Y=1)$ and 
  $\widehat{f}(g(\x)|Y=0)$.
  \ForAll{$i \in \{1,\ldots,B\}$}
   \State Sample  $W^{(0)}_1,\ldots,W^{(0)}_{n_0} 
   \sim \widehat{f}(g(\x)|Y=0)$ and 
   $W^{(1)}_1,\ldots,W^{(1)}_{n_1} \sim 
   \widehat{f}(g(\x)|Y=1)$
   \State Sample $W^{(U)}_1,\ldots,W^{(U)}_{n_U} \sim
   \widehat{\theta}\widehat{f}(g(\x)|Y=1)+
   (1-\widehat{\theta})\widehat{f}(g(\x)|Y=0)$
   \State Compute $T_i$, the test statistic based on
   the labeled and unlabeled samples:
   \begin{align*}
    \{(W_i^{(0)},Y_i=0)\}_{i=1}^{n_0} 
    \cup \{(W_i^{(1)},Y_i=1)\}_{i=1}^{n_1};
    && \{W_1^{(U)},\ldots, W_{n_u}^{(U)}\}
   \end{align*}
  \EndFor
  \State \textbf{return} 
  $\frac{1}{B}\sum_{i=1}^B \I\left(T_{\mbox{obs}}\geq T_i\right) $
  %(sum(test.MC >= test.statistic)+1)/(B+1)
 \end{algorithmic}
\end{algorithm}

%On the regression case,
%one may use e.g. \citet{hoeffding1948non}
%to test if $g(\X)$ is independent of $Z$ given $Y=i$, $i=0,1$.
Note that our test is different from
those proposed by \citet{Saerens2002} and \citet{Lipton2018}.
While our test evaluates whether the prior shift assumption is
reasonable for the observed data,
the above tests \emph{assume prior shift} and 
evaluate whether the prevalence of 
positive labels in the unlabeled sample is
the same as that in the labeled sample, that is,
$\P(Y=1|S=0) = \P(Y=1|S=1)$.

The following subsection performs 
several experiments to test 
the performance of Algorithm \ref{alg:MC} and
of the ratio estimators which were discussed in
previous subsections.

\subsection{Experiments}
\label{sec:expRatio}

Next, we compare the errors of 
the ratio estimator and of
the classify and count estimator based on
the estimator $g$ when
using various methods for obtaining $g$. We also include comparisons with the EM methods by \citet{Saerens2002}.
Table \ref{tab:methods} summarizes all 
the variants that were tested.

\begin{table}[!h]
 \centering
 \resizebox{\textwidth}{!}{
 \begin{tabular}{|l|l|l|}
  \hline 
  Estimator class
  & Specific method 
  & Criteria for choosing $g(x)$ \\
  \hline 
  Classify and count &
  & Logistic regression (LR), 
  $k$-NN, random forest (RF). \\ 
  \hline
  \multirow{2}{*}{Ratio} & \citet{Forman2006} 
  & Logistic regression (LR), $k$-NN, 
  random forest (RF). \\
  & \citet{Bella2010} & 
  Logistic regression (LR), $k$-NN, 
  random forest (RF). \\
  & RKHS & 
  Linear kernel (Linear), Gaussian kernel (Gauss).  \\
  \hline
    & EM \citep{Saerens2002} & 
  Logistic regression (LR), $k$-NN, 
  random forest (RF).  \\
  \hline 
 \end{tabular}}
 \caption{Methods compared in the experiments.}
 \label{tab:methods}
\end{table}

We compare the above methods in five data sets: 
Candles \citep{Freeman2013,Izbicki2013}, 
Bank Marketing \citep{Moro2011}, 
SPAM e-mail \citep{Blake1998}, 
Wisconsin Breast Cancer \citep{Mangasarian1990} and 
Blocks Classification \citep{Donato1996}.
Each database was transformed into 
a prior shift problem by choosing at random
$n_1$ ($n_0$) instances among the ones labeled as $1$ ($0$)
to be marked as labeled instances and by 
choosing $n_U$ instances to be marked as unlabeled. 
Each unlabeled unit is taken
with probability $\theta$ randomly among
the instances labeled as $1$ in the original data set and
with probability $1-\theta$ among those labeled as $0$. 
The quantification sample sizes used in
each of these data sets are described in
Table \ref{tab:expinf}.
For all data sets we let $\theta$ vary in 
$\{0.1;0.2;0.3;0.4;0.5\}$ and repeated
the generation and testing $100$ times for
each of the 11 methods in Table \ref{tab:methods}.

\begin{table}[h]
 \centering
 \begin{tabular}{rlrrrr}
  \hline
  & data set & $n_U$ & $n_L$ & $n_1$ & $n_0$ \\ 
  \hline
  1 & cancer & 100 & 300 & 150 & 150 \\ 
  2 & candles & 300 & 300 & 150 & 150 \\ 
  3 & block & 800 & 300 & 150 & 150 \\ 
  4 & spam & 2000 & 300 & 150 & 150 \\ 
  5 & bank & 10000 & 300 & 150 & 150 \\ 
  \hline
 \end{tabular}
 \caption{Sample sizes for each data set.}
 \label{tab:expinf}
\end{table}

Figure \ref{fig::error} represents
the average of the mean squared error 
(MSE; red point) and
a confidence interval for the MSE
(vertical blue bar)
for each setting and method\footnote{For the sake of visualization, we omit all estimators based on $K$-NN on this plot. Figure  \ref{fig::error_all} in Appendix \ref{ap:figures} contains all methods.}. 
%We also numerically evaluate  the average of the MSE in order to discover in which method this value is minimal within each experiment setting. 
Figure \ref{fig::count} shows the number of experiments in which each method had the best average MSE, considering all data sets and $\theta$ values.
These plots indicate that 
the ratio estimator generally performs better
than the classify and count estimator
for all choices of $g$.
The main exception to this rule occurs when
$\theta \approx 0.5$ and hence there is no prior shift.
Also, the method in \cite{Bella2010} performs 
better than the one in \cite{Forman2006} in 
essentially all scenarios. This suggests that
soft classifiers might lead to 
better ratio estimators than hard classifiers. Moreover, 
the best performance is usually achieve when  $g$  is based on Random Forest,
which corroborates that choosing  a good classifier is key to having a good estimate of $\theta$.
The EM method was found to be very competitive in general, although
the ratio estimators had better performance in some cases (e.g., for the bank data set).
Finally, the RKHS approach is a competitive method,
especially when using the Gaussian kernel.
For additional figures related to this experiment see Appendix
\ref{ap:figures}.

\begin{figure}
 \centering
 \makebox[\textwidth]{
  \includegraphics[page=1,scale=0.5]
  {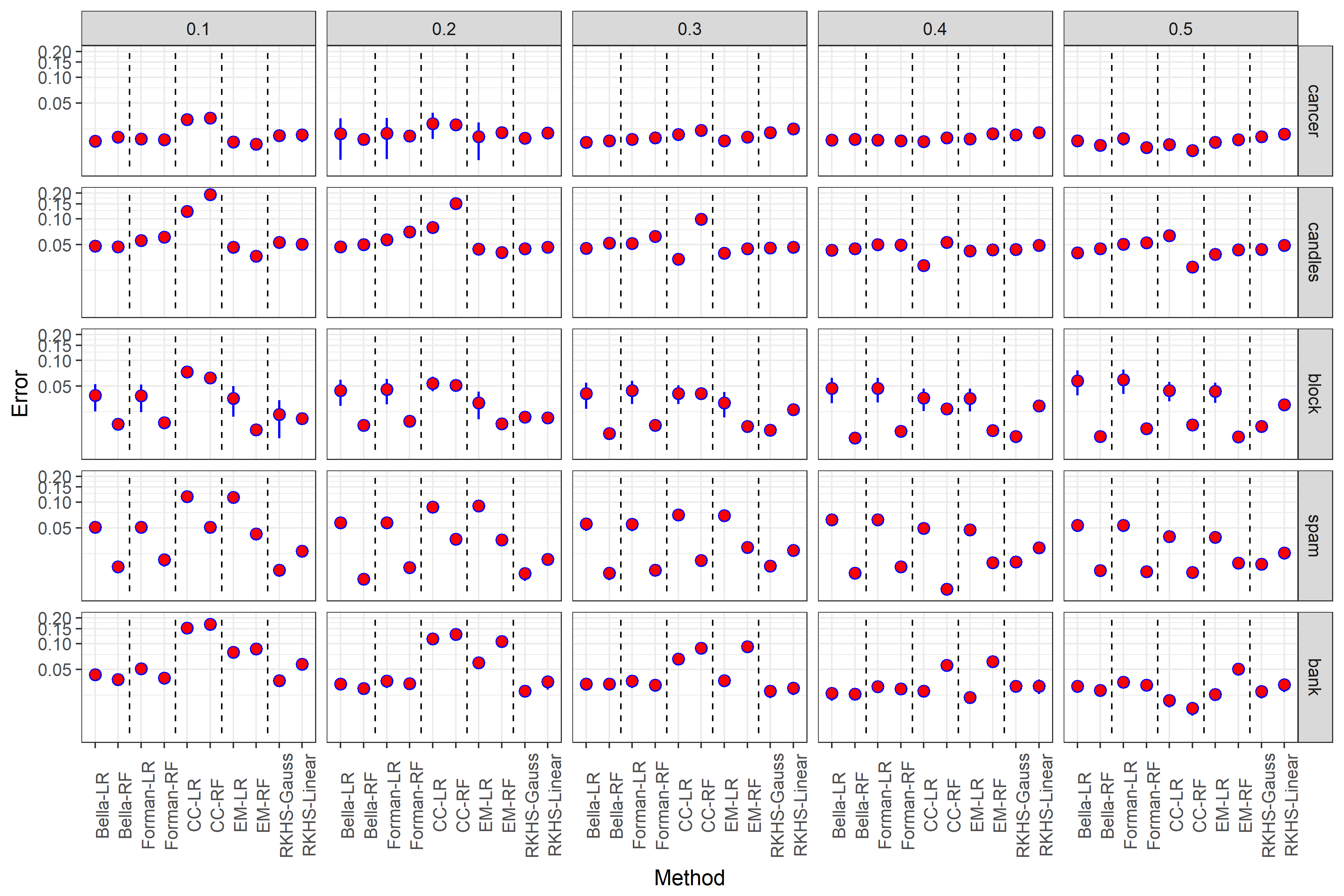}
 }
 
 \caption{Root mean square deviation of 
 each method by setting in
 logarithmic scale.}
 \label{fig::error}
\end{figure}

\begin{figure}
 \centering
 \makebox[\textwidth]{
  \includegraphics[page=1,scale=0.45]
  {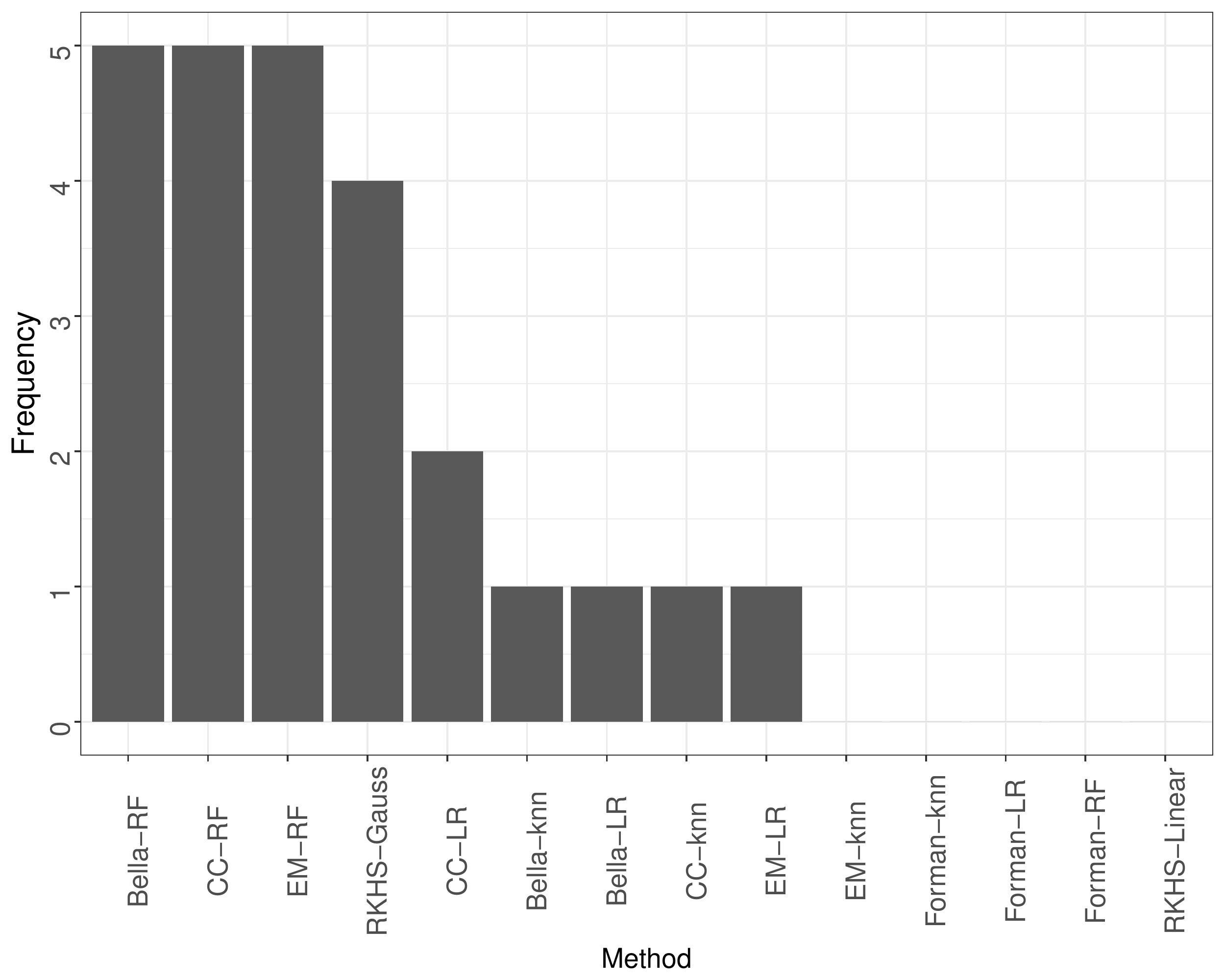}
 }
 \caption{Number of times in which each method presented best MSE.}
 \label{fig::count}
\end{figure}

\begin{figure}
 \centering
 \makebox[\textwidth]{
  \includegraphics[page=1,scale=0.51]
  {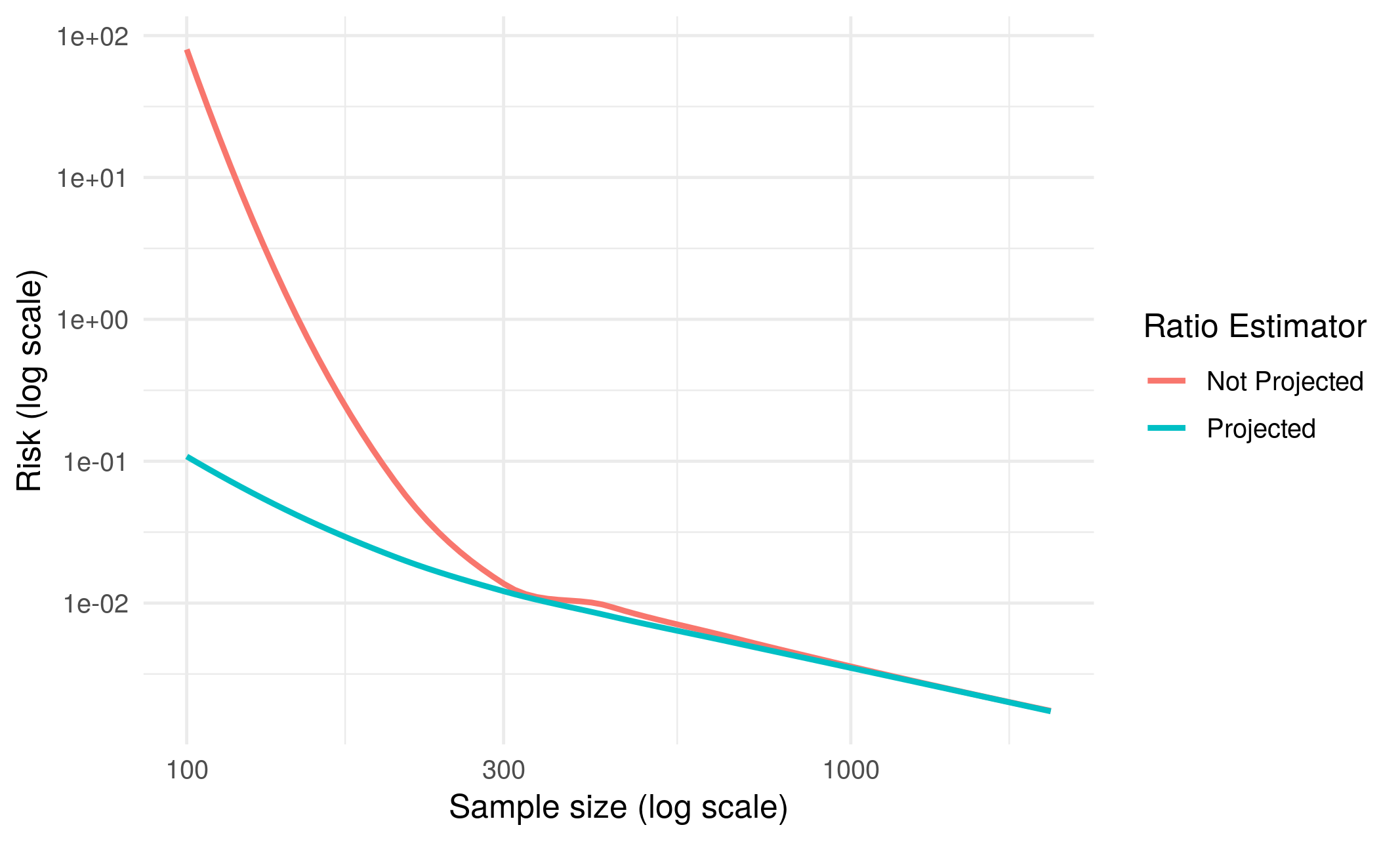}
 }
 \caption{Mean squared error of the ratio estimator for the multiclass problem.}
 \label{fig::multiclass}
\end{figure}

For all of ratio estimators, data sets and
values for $\theta$ above,
we construct confidence intervals for
$\theta$ based on \autoref{thm:ratio-clt}.
We find that in all but one scenario the
empirical coverage was at least as high as
the specified value of 95\%.
The empirical coverage in the exception was 94\%.
The intervals constructed using
\autoref{thm:ratio-clt} seem to be conservative.

Next, we simulate data using the following multiclass setting:
$\X|Y=y \sim N(\mu_y,\Sigma)$ with $\Sigma=I_{10}$, $\mu_1=(0,\ldots,0)$,
$\mu_2=(0.75,\ldots,0.75)$,
$\mu_3=(1.25,\ldots,1.25)$,
$\P(Y=1|S=1)=0.2$, $\P(Y=2|S=1)=0.3$
$\P(Y=1|S=0)=0.25$, and $\P(Y=2|S=0)=0.10$.
We use a multivariate logistic regression to compute $g_1(\x)=\widehat{P}(Y=1|\x,S=1)$
and $g_2(\x)=\widehat{P}(Y=2|\x,S=1)$.
\autoref{fig::multiclass} indicates that the  mean squared error of the multiclass ratio
estimator goes to zero as the  sample size increases. Moreover, it shows that projecting the raw estimator to the simplex improves the convergence, especially for small sample sizes.

We also evaluate the
power of the weak prior shift test in
Section \ref{sec:hypothesis}.
In order to test the weak prior shift,
we generate data according to 4 scenarios,
which are presented in table
\ref{tab::simulationPowerPrioShift}.
In all of these scenarios,
the weak prior shift assumption holds for
a single value of $\gamma$.
Figure \ref{fig::power} presents the power of
the weak prior shift test in each scenario
using a level of significance of $\alpha=5\%$.
Besides the test achieving the level of
$5\%$ when weak prior shift holds,
it also has a high power whenever the
marginal distribution of $g(\X)$ differs over
the labeled and over the unlabeled data.
The reason why such test presents minima when the weak priori shift assumption does not hold
is described in Example \ref{ex:minima}.

\begin{table}
 \centering
 \label{tab::simulationPowerPrioShift}
 \begin{tabular}{|l|l|}
  \hline
  $\begin{array}{l} \mbox{\textbf{Gaussian}} \\
   g(\x)|S=1,Y=0 \sim \mbox{N}(0,1),\\
   g(\x)|S=1,Y=1 \sim \mbox{N}(2,1), \\
   g(\x)|S=0,Y=0 \sim \mbox{N}(\gamma,1) \\
   g(\x)|S=0,Y=1 \sim \mbox{N}(2,1), \\
   \P(Y=1|S=0) = 0.6, \\
   \P(Y=1|S=1) = 0.2
  \end{array}$ &
  $\begin{array}{l}
   \mbox{\textbf{Exponential}} \\
   g(\x)|S=1,Y=0 \sim \mbox{Exp}(1),\\
   g(\x)|S=1,Y=1 \sim \mbox{Exp}(5), \\
   g(\x)|S=0,Y=0 \sim \mbox{Exp}(\gamma) \\
   g(\x)|S=0,Y=1 \sim \mbox{Exp}(5), \\
   \P(Y=1|S=0) = 0.6, \\
   \P(Y=1|S=1) = 0.2
  \end{array}$ \\ \hline
  $\begin{array}{l}
   \mbox{\textbf{Gaussian-Exponential}} \\
   g(\x)|S=1,Y=0 \sim \mbox{N}(1,1),\\
   g(\x)|S=1,Y=1 \sim \mbox{Exp}(1), \\
   g(\x)|S=0,Y=0 \sim \mbox{N}(\gamma,1) \\
   g(\x)|S=0,Y=1 \sim \mbox{Exp}(1), \\
   \P(Y=1|S=0) = 0.6, \\
   \P(Y=1|S=1) = 0.2
  \end{array}$ &   
  $\begin{array}{l}
   \mbox{\textbf{Beta}} \\
   g(\x)|S=1,Y=0 \sim \mbox{Beta}(1,1),\\
   g(\x)|S=1,Y=1 \sim \mbox{Beta}(1,10), \\
   g(\x)|S=0,Y=0 \sim \mbox{Beta}(\gamma,1) \\
   g(\x)|S=0,Y=1 \sim \mbox{Beta}(1,10), \\
   \P(Y=1|S=0) = 0.6, \\
   \P(Y=1|S=1) = 0.2
  \end{array}$ \\
  \hline
 \end{tabular}
 \caption{Scenarios used for testing
 the weak prior shift assumption.}
\end{table}

\begin{figure}
 \centering
 \includegraphics[page=1,scale=0.4]
 {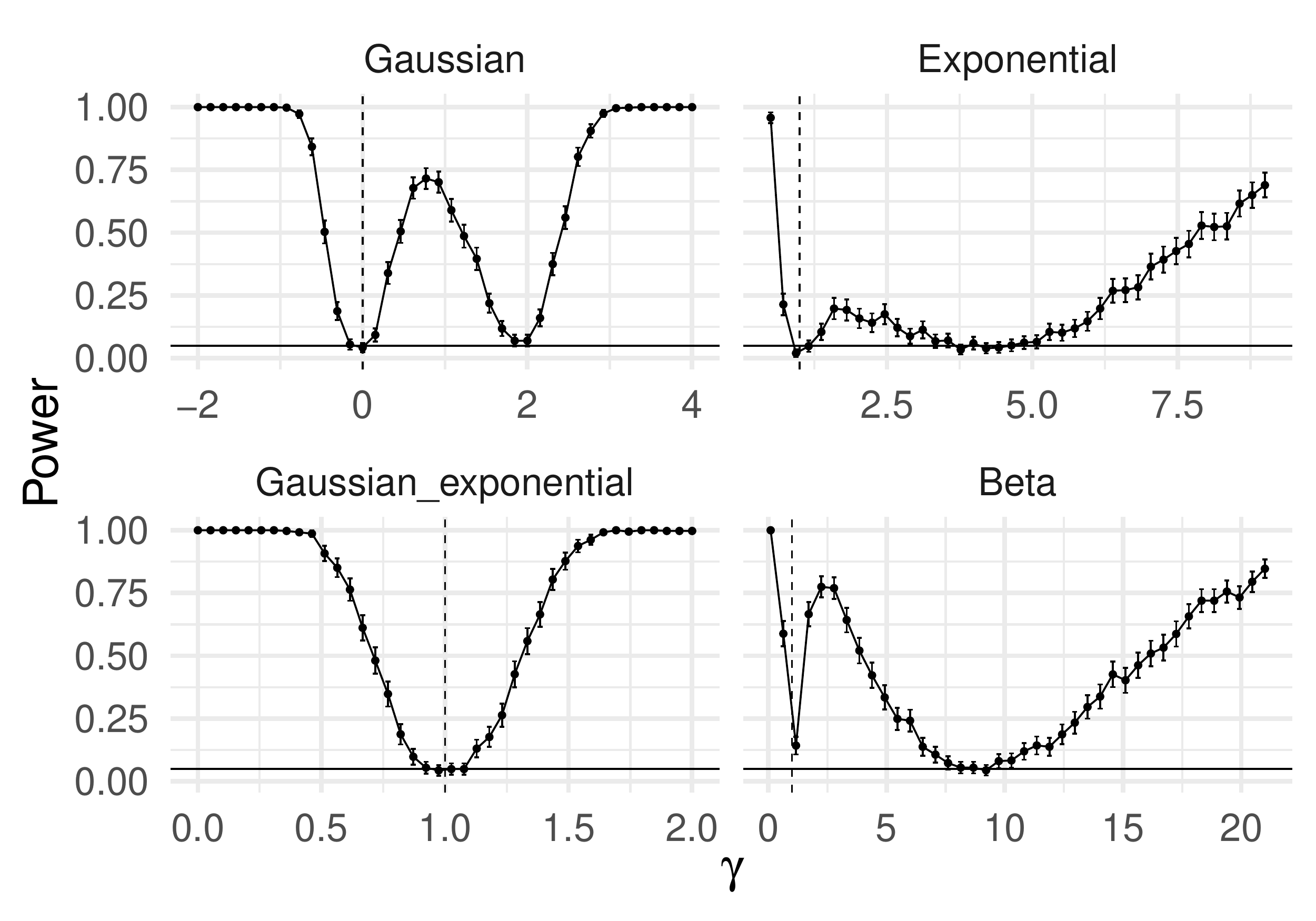}
 \caption{Power of the weak prior shift test at
 level $\alpha=5\%$. The dashed vertical lines indicate 
 the value of $\gamma$ for which 
 the weak prior shift holds.}
 \label{fig::power}
\end{figure}

The following section discusses extensions of
the ratio estimator to scenarios that are
more general than the standard quantification problem.

\section{Extensions of the quantification problem}
\label{sec:extensions}

\subsection{Combined estimator}
\label{sec:combined}
 
Sometimes, a few labels are available in
the target population $(S=0)$.
Let $A^{*}_0 \subset A_0$ denote
the indices of these labeled sample instances.
In this scenario, it is possible to 
obtain an estimate of $\theta$ that 
combines the ratio estimator with
the additional labels which are available.
The labeled estimator of $\theta$ is
defined as:
\begin{align*}
 \widehat{\theta}_{L} 
 &:= \dfrac{1}{|A^{*}_0|} \sum_{i \in A^{*}_0} Y_i
\end{align*}

It is possible to better estimate $\theta$ by
combining the labeled estimator, 
$\widehat{\theta}_L$, with the ratio estimator,
$\widehat{\theta}_R$.
We propose to combine these estimators 
by means of a convex combination,
\begin{align}
\label{eq::combined}
 \widehat{\theta}_C
 &= w\widehat{\theta}_R + (1-w)\widehat{\theta}_L,
\end{align}
which we name the \emph{combined} estimator.
The following theorem provides
an insight on how to choose
$w$.
\begin{theorem}
\label{thm:combined}
Under Assumptions
 \ref{assump::iidRelaxed}, 
 \ref{assump::cool1},
 \ref{assump::separability} and 
 \ref{assump::ratio-clt},
 the value of $w$ that minimizes the mean squared error of the combined estimator (\autoref{eq::combined}) is 
$$w^* = MSE[\widehat{\theta}_L]\times (MSE[\widehat{\theta}_L] + MSE[\widehat{\theta}_R])^{-1}.$$
\end{theorem}

In practice, $MSE[\widehat{\theta}_L]$ and
$MSE[\widehat{\theta}_R]$ 
need to be estimated.
Note that $MSE[\widehat{\theta}_L] = \theta(1-\theta)\times |A_0^*|^{-1}$ and $MSE[\widehat{\theta}_R]$ is
given by \autoref{thm::EQMPriorShift}.
We therefore
use
$\widehat{w}
=\widehat{MSE}[\widehat{\theta}_L] \times 
(\widehat{MSE}[\widehat{\theta}_L]
+\widehat{MSE}[\widehat{\theta}_R])^{-1}$,
where $\widehat{MSE}[\widehat{\theta}_L]$ and
$\widehat{MSE}[\widehat{\theta}_R]$ are obtained
by substituting the parameters in 
$MSE[\widehat{\theta}_L]$ and 
$MSE[\widehat{\theta}_R]$ by 
their corresponding empirical averages.

\begin{figure}
 \centering
 \makebox[\textwidth]{
  \includegraphics[page=3,scale=0.51]{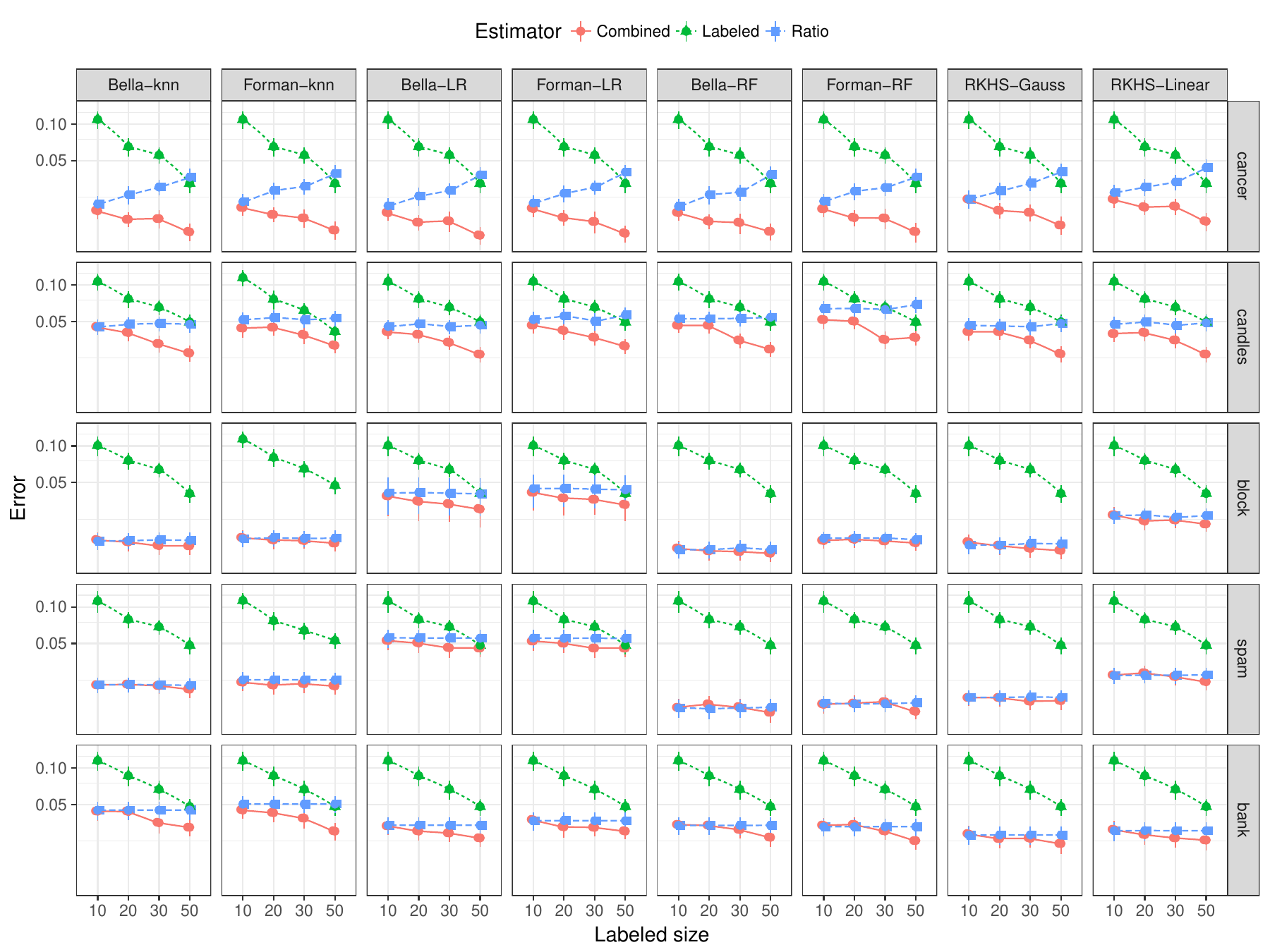}}
 \caption{Root mean square deviation in
 logarithmic scale for
 each data set, method and estimator by
 size of the labeled sample and using $\theta=0.3$}
 \label{fig::comb1}
\end{figure}

We evaluate the combined estimator
under the same scenarios used for 
the ratio estimator in Section \ref{sec:expRatio} and
using $\theta = 0.3$. 
For each scenario, we consider $10$, $20$, $30$,
$40$ or $50$ available labels from
the target population.
Figure \ref{fig::comb1} presents the errors for 
each setting scenario and number of available labels
in the target population.\footnote{Similar plots (with similar conclusions) for $\theta \in \{0.1,0.2,0.4,0.5\}$ can be found in  Figures \ref{fig::comb2}---\ref{fig::comb5} in Appendix \ref{ap:figures}.}
When one of $\widehat{\theta}_L$ and 
$\widehat{\theta}_R$ has an error which is 
much lower than the other, than this lowest error
is comparable to that of the combined estimator.
Also, when $\widehat{\theta}_L$ and
$\widehat{\theta}_R$ have similar errors, then
the error of the combined estimator is
approximately $\sqrt{2}^{-1}$ times 
this common error.
These results indicate that
a few labels from the target population
can improve the estimation of $\theta$.

\subsection{Regression quantification}
\label{sec::regression}

As a generalization of the quantification problem,
one might be interested on how the prevalence of $Y$
in the target population varies 
according to a new set of covariates, $\Z$.
For example, suppose that a company implements
a program of continuous improvement for
one of its products.
In order to measure the effects of the program, 
it is necessary to evaluate how the
proportion of positive reviews for the product
varies over time. This problem fits into
the generalization of the quantification problem
when taking $\Z$ to be the date at which
each review was posted. 
This problem is often called 
sentiment analysis \citep{Wang2012} and
is usually solved using a
classify and count approach, which
suffers from the same downsides as the ones
discussed in the standard quantification problem.
Other approaches can be found in \citet{Hofer2013}.

The ratio estimator can be extended 
to this regression setting.
Let the new sample instances be
$(\X_1,\Z_1,Y_1,S_1),\ldots,
(\X_n,\Z_n,Y_n,S_n)$,
where $\X$, $Y$ and $S$ have
the same interpretation as in
the quantification problem and
$\Z$ is the new covariate of interest.
The goal in the new setting is to estimate
\begin{align*}
 \theta(\z) := \P(Y=1|S=0,\z),
\end{align*}
the proportion of positive
labels in the target population when $\Z=\z$.
In order to estimate $\theta(\z)$, it
is necessary to make additional assumptions
on how $\Z$ relates to the other variables in
the quantification problem.
One such assumption is presented below.

\begin{assumption}
 \label{assump::condCovIndep}
 $g(\X)$ is stochastically independent of $\Z$ 
 conditionally on $Y$ and $S$.
\end{assumption}

In the scenario in which $\X$ are 
written reviews of products and $Z$ is time,
Assumption  \ref{assump::condCovIndep} states that,
if the label of a product is known, then
the time at which the label was given 
does not affect the written review.
This assumption motivates the definition of
the regression ratio estimator.

\begin{definition}
 \label{defn:reg_ratio}
 The untrimmed regression ratio estimator,
 $\widehat{\theta}_{URR}(\z)$, is 
 \begin{align*}
  \widehat{\theta}_{URR}(\z)
  &= \frac{\widehat{\E}[g(\X)|S=0,\z] -
  \widehat{\E}[g(\X)|Y=0,S=1]}
  {\widehat{\E}[g(\X)|Y=1,S=1] -
  \widehat{\E}[g(\X)|Y=0,S=1]},
 \end{align*}
 where $\widehat{\E}[g(\X)|Y=0,S=1]$ and
 $\widehat{\E}[g(\X)|Y=1,S=1]$ are 
 the same empirical averages as in
 Definition \ref{defn:ratio} and
 $\widehat{\E}[g(\X)|S = 0 ,\z]$
 is an estimate of the regression function $\E[g(\X)|S = 0 ,\z]$.
 For instance $\widehat{\E}[g(\X)|S = 0 ,\z]$ could be
 the Nadaraya-Watson regression estimator \citep{Nadaraya1964}
 based on the target population for $g(\X)$ given $\Z$.
 The regression ratio estimator,
 $\widehat{\theta}_{RR}(\z)$, is
 $\max(0,\min(1,\widehat{\theta}_{URR}(\z)))$.
\end{definition}

Next, we derive an upper bound on
the rate of convergence of the regression ratio estimator.

\begin{theorem}
 \label{thm:rr-mse}
 Under Assumptions \ref{assump::iidRelaxed},
 \ref{assump::cool1} and \ref{assump::condCovIndep},
 \begin{align*}
  \E\left[\left(\widehat{\theta}_{RR}(\Z)
  -\theta(\Z)\right)^2\bigg|S_1^n\right]
  &\leq \sO\left(\max\left(
  \E\left[(\widehat{\E}[g(\X)|S=0,\Z]
  -\E[g(\X)|S=0,\Z])^2|S_1^n\right], n_{L}^{-1}
  \right)\right)
 \end{align*}
\end{theorem}

\autoref{thm:rr-mse} shows that 
the integrated mean squared error of
$\widehat{\theta}_{RR}$ depends both on $n_L$ 
and on the integrated mean squared error of
$\widehat{\E}[g(\X)|S=0,\Z]$ with respect to
the regression function, $\E[g(\X)|S=0,\Z]$.
If one uses standard nonparametric methods to
estimate $\E[g(\X)|S=0,\Z]$, then it is possible
to prove the consistency of $\widehat{\theta}_{RR}$
under weak additional assumptions. For example,
if in the target population the pairs 
$(g(\X_1),\Z_1),\ldots,(g(\X_{n_U}),\Z_{n_U})$ are
i.i.d., the regression function $\E[g(\X)|S=0,\z]$
is sufficiently smooth over $\z$ and 
$\widehat{\E}[g(\X)|S=0,\z]$ is
the Nadaraya-Watson kernel estimator with bandwidth
$h_{n_U} = O(n_{U}^{-1/5})$, then it follows
\citep{Wasserman2006}[p.73] that
\begin{align*}
 \E\left[\left(\widehat{\theta}_{RR}(\Z)
 -\theta(\Z)\right)^2\bigg|S_1^n\right]
 &\leq \sO\left(\max(n_U^{-4/5}, n_L^{-1})\right)
\end{align*}

In the equation above, the rate of convergence
over $n_U$ is slower than the one obtained in
\autoref{thm:mse-trimmed}.
This is expected, since the rates of convergence
of nonparametric estimators are typically
slower than those of sample means.

Besides showing the consistency of
the regression ratio estimator, we
also show that it outperforms the
classify and count method in
artificial data sets.
Specifically, we run the
regression ratio estimator using
the Nadaraya-Watson kernel estimator and
generate data sets under 
the following specifications:
\begin{itemize}
 \item $Z \sim U(0,1)$
 \item $\P(Y=1|S=1,Z=z) = 0.5$
 \item $\P(Y=1|S=0,Z=z) = 0.5(\sin (2zk\pi) + 1) $, 
 for $k \in \{1,2\}$
 \item $X|Y = 0 \sim N(\mu, 1)$ and 
 $X|Y = 1 \sim N(-\mu, 1)$, 
 for $\mu \in \{0.5,1,1.5,2\}$
 \item $n_{L} = n_{U} = 10^3$
 \item $g(X)$ is the Bayes classifier, i.e. $g(X) = \I(X>0)$
\end{itemize}

\begin{figure}
 \centering
 \makebox[\textwidth]{
  \includegraphics[page=2,scale=0.51]{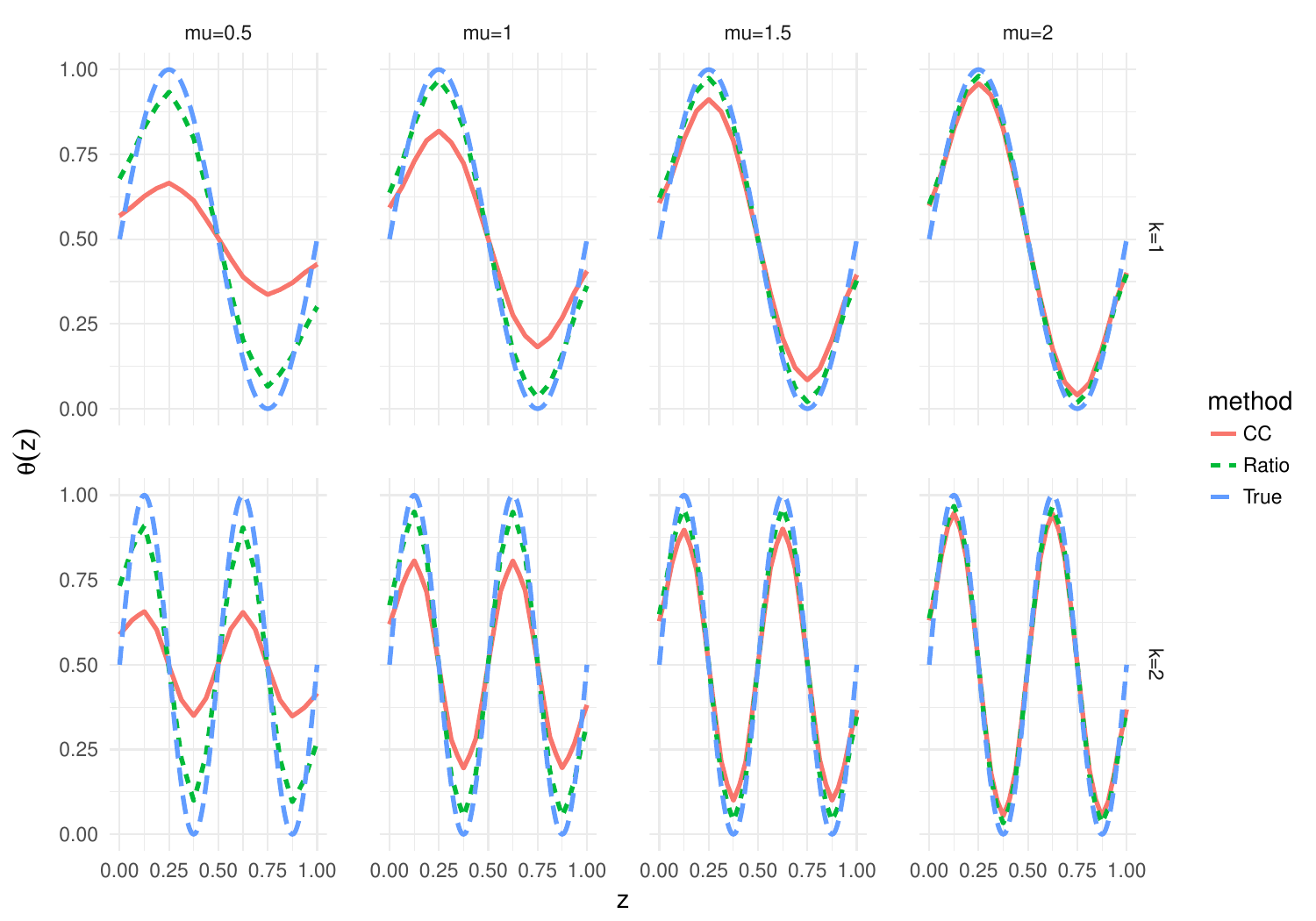}
 }
 \caption{Average of the fitted regression in 
 each setting.}
 \label{fig::reg1}
\end{figure}

For each combination of $k$ and $\mu$, 
400 independent data sets were generated.
Figure \ref{fig::reg1} presents the average curve
fitted for $\theta(\z)$ using the regression ratio
and classify and count estimators. 
One can observe that, while for small values of $\mu$
the ratio regression outperforms 
the classify and count estimator,
for large values of $\mu$ both estimators are similar.
This occurs because when $\mu$ is large,
the classification problem of determining 
the value of $Y$ is easier and 
both methods perform well.
One can also observe from Figure \ref{fig::reg1} that
the classify and count method performs worse than
the regression ratio estimator because its fit is
generally smoother than the true regression curve.
Figure \ref{fig::reg2} summarizes the 
mean squared error of each method in each setting.
The regression ratio  estimator leads to 
substantial  improvements over 
the classify and count method.

\begin{figure}
 \centering
 \makebox[\textwidth]{
  \includegraphics[page=1,scale=0.51]{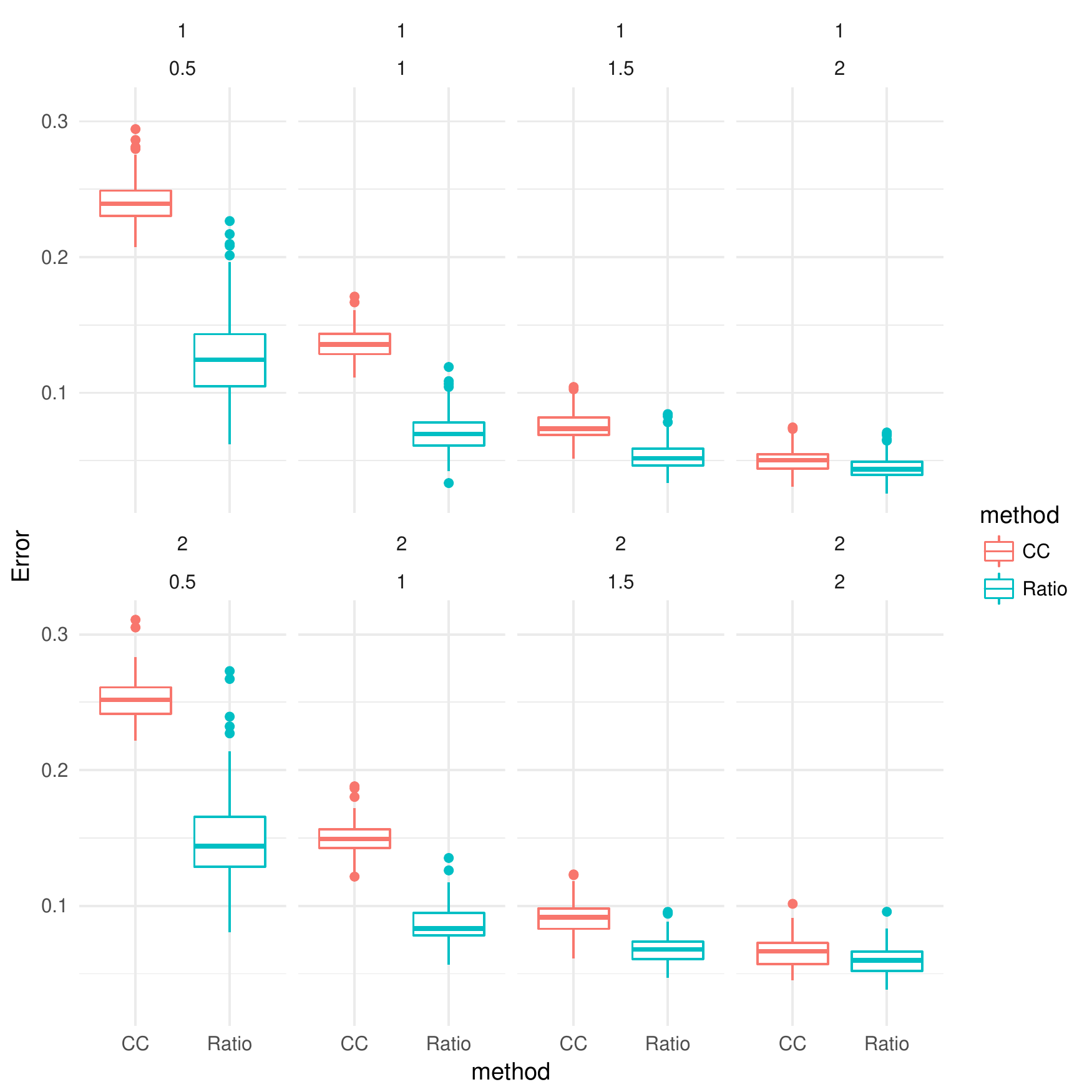}
 }
 \caption{Boxplots of the
 root mean square deviation in each setting.}
 \label{fig::reg2}
\end{figure}

\section{Final remarks}

We present the ratio estimator for the 
problem of quantification,
show that it is approximately minimax
under the prior probability shift assumption,
and provide an hypothesis test for this assumption.
Since the methods in \citet{Forman2006}
and \citet{Bella2010} are particular
instances of the ratio estimator, 
it follows that they are also approximately minimax.
The lower bound on the risk that
we derive is of independent interest and 
can be used to investigate the optimality of 
other quantification methods.
We also derive the limiting distribution
of the ratio estimator, which allows
the derivation of a ratio estimator based on
Reproducing Kernel Hilbert Spaces and of
confidence intervals for the quantification problem.
A simulation study shows that 
the ratio estimator based on 
Reproducing Kernel Hilbert spaces is
a competitive new alternative.

Besides the above results, we also generalize
the ratio estimator to two other scenarios.
In the first one, we consider the case in
which some labels are available in
the target population. 
The combined estimator uses these labels and
the ratio estimator to obtain a larger effective
sample size than the ratio estimator.
In the second scenario, we consider the
prevalence of positive labels varies according
to a new variable, $\Z$. We show that,
under Assumption \ref{assump::condCovIndep},
the regression ratio estimator can
be made consistent for $\theta(\Z)$.
A still unresolved issue is how much
it is possible to relax
Assumption \ref{assump::condCovIndep}
while still being able to learn
$\theta(\Z)$.

\acks{This work was partially supported by
	FAPESP grant  2017/03363-8
	and CNPq grant 306943/2017-4.
This study was financed in part by the Coordena\c{c}\~ao de Aperfei\c{c}oamento de Pessoal de N\'ivel Superior - Brasil (CAPES) - Finance Code 001.
}

\appendix

\section{Proofs}
\label{ap:proofs}

\begin{lemma}
 \label{lemma:lower-bound}
 Let $\XX_{U} = (\X_u)_{u \in A_0}$,
 $\Y_U = (Y_u)_{u \in A_0}$,
 $\XX_{L} = (\X_l)_{l \in A_1}$, and
 $\Y_L = (Y_l)_{l \in A_1}$.
 Under Assumptions \ref{assump::iidRelaxed} and
 \ref{ass::priorShift},
 if $\theta \sim \text{Uniform}(0,1)$, then
 $\E[\V[\theta|\XX_{U},\XX_{L},\Y_L]] 
 \geq \Omega(n_{U}^{-1})$. 
 Under the same assumptions, if
 $\P(\theta = 0.5-n_{L}^{-1}) 
 =\P(\theta = 0.5+n_{L}^{-1}) = 0.5$,
 $f_{0} = \text{Bernoulli}(0)$,
 $\alpha \sim \text{Uniform}(\epsilon^*, 1)$
 $f_{1}|\alpha = \text{Bernoulli}(\alpha)$, then
 $\E[\V[\theta|\XX_{U},\XX_{L},\Y_L]] 
 \geq \Omega(n_{L}^{-1})$.
\end{lemma}

\begin{proof}
 It follows from Assumptions
 \ref{assump::iidRelaxed} and
 \ref{ass::priorShift} that the
 dependency relations between data
 and parameters can be represented by
 figure \ref{fig:quantmodel}.
 \begin{figure}[h]
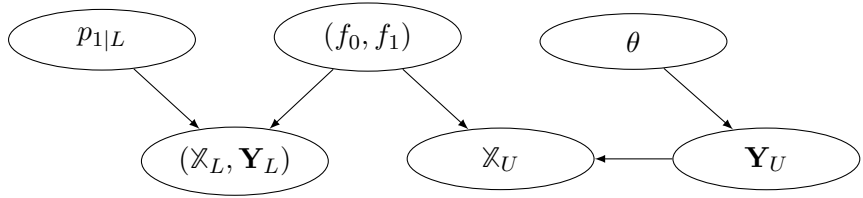

  \centering
  \quantmodel
  \caption{Dependency relations between data and
  parameters in the prior shift model.}
  \label{fig:quantmodel}
 \end{figure}
 
 If $\theta \sim U(0,1)$, then
 \begin{align}
  \label{eq:lower-2}
  \E[\V[\theta|\XX_U,\XX_L,\Y_L]|S_1^n] &\geq
  \E[\V[\theta|\XX_U,\XX_L,\Y_L,\Y_U]|S_1^n]
  \nonumber \\
  &= \E[\V[\theta|\Y_U]|S_1^n] 
  & \text{fig. \ref{fig:quantmodel}} \nonumber \\
  &= \Omega(n_U^{-1})
  & \Y_U|\theta \text{ i.i.d. Bernoulli}(\theta) 
 \end{align}
 
 Next, let 
 $\P(\theta=0.5-n_{L}^{-0.5}) =
 \P(\theta=0.5+n_{L}^{-0.5}) = 0.5$,
 $f_0 = \text{Bernoulli}(0)$
 $\alpha \sim U(\epsilon,1)$ and
 $f_1|\alpha = \text{Bernoulli}(\alpha)$.
 Define $\XX_{L,1} = (\X_l)_{l \in A_{1,1}}$
 and note that
 \begin{align}
  \label{eq:lower-3}
  f(\alpha|\XX_L,\Y_L)
  &\propto \alpha^{n_{L,1} \bar{\XX}_{L,1}}
  (1-\alpha)^{n_{L,1}(1-\bar{\XX}_{L,1})}
  \I(\alpha \in (\epsilon,1))
 \end{align}
 Let $\lambda := \theta\alpha$ and
 observe that, for every $\lambda \in \left(\epsilon(0.5+n_L^{-0.5}), 0.5-n_L^{-0.5}\right]$,
 \begin{align}
  \label{eq:lower-4}
  \P(\theta=0.5+n_L^{-0.5}|\lambda,\XX_L,\Y_L)
  &= \frac{f_{\alpha}(\lambda(0.5+n_L^{-0.5})^{-1}
  |\XX_L,\Y_L)}
  {f_{\alpha}(\lambda(0.5+n_L^{-0.5})^{-1}
  |\XX_L,\Y_L)
  +f_{\alpha}(\lambda(0.5-n_L^{-0.5})^{-1}
  |\XX_L,\Y_L)} \nonumber \\
  &= \frac{1}{1+\frac
  {\left(\frac{\lambda}{0.5-n_L^{-0.5}}\right)^
  {n_{L,1} \bar{\XX}_{L,1}}
  \left(1-\frac{\lambda}{0.5-n_L^{-0.5}}\right)^
  {n_{L,1}(1-\bar{\XX}_{L,1})}}
  {\left(\frac{\lambda}{0.5+n_L^{-0.5}}\right)^
  {n_{L,1} \bar{\XX}_{L,1}}
  \left(1-\frac{\lambda}{0.5+n_L^{-0.5}}\right)^
  {n_{L,1}(1-\bar{\XX}_{L,1})}}}
  & eq. \ref{eq:lower-3} \nonumber \\
  &= \frac{1}{1+
  \left(1 + \frac{4}{n_L^{0.5}-2}\right)^
  {n_{L,1}}
  \left(1 - \frac{4}
  {(1-2\theta \alpha)n_L^{0.5}+2}\right)^
  {n_{L,1}(1-\bar{\XX}_{L,1})}}
  & \lambda = \theta\alpha
 \end{align}

 Let $\gamma_{n_L} := 
 \P(\theta=0.5+n_L^{-1}|\lambda,\XX_L,\Y_L)$.
 Note that $\frac{n_{L,1}}{n_L} \convas p_{1|L}$
 and $\bar{\XX}_{L,1} \convas \alpha$.
 Therefore, since $|\theta-0.5| \leq n_{L}^{-0.5}$,
 it follows from eq. \ref{eq:lower-4} that
 $\gamma_{n_L}$ converges a.s. to a
 quantity between 
 $\frac{1}{1+
 \exp\left(\frac{-8\alpha p_{1|L}}{1-\alpha}\right)}$ and
 $\frac{1}{1+
 \exp\left(\frac{8\alpha p_{1|L}}{1-\alpha}\right)}$.
 That is, 
 
 \begin{align}
  \label{eq:lower-5}
  \E[\gamma_{n_L}(1-\gamma_{n_L})|S_{1}^n]
  &\geq \Omega(1).
 \end{align}
 
 Note that $\bar{\XX}_U$ is
 sufficient for $(\theta,\alpha)$ and
 $\bar{\XX}_U$ converges a.s. to $\lambda$.
 Therefore,
 \begin{align}
  \label{eq:lower-6}
  \E[\V[\theta|\XX_U,\XX_L,\Y_L]|S_1^n] 
  &= \E[\V[\theta|\bar{\XX}_U,\XX_L,\Y_L]|S_1^n] 
  \nonumber \\
  &\geq \E[\V[\theta|\lambda,\XX_L,\Y_L]|S_1^n] 
  \nonumber \\
  &\geq 4n_L^{-1}\E\left[\gamma_{n_L}(1-\gamma_{n,L}) \I\left(\lambda \in \left(\epsilon(0.5+n_L^{-1}), 0.5-n_L^{-1}\right]\right)|S_1^n\right]
 \end{align}
 
 Since $\P(\lambda \in \epsilon(0.5+n_L^{-1}),0.5-n_L^{-1}) \geq 0.5$, it follows from eqs.
 \ref{eq:lower-5} and \ref{eq:lower-6} that
 \begin{align}
  \label{eq:lower-7}
  \E[\V[\theta|\XX_U,\XX_L,\Y_L]|S_1^n] 
  \geq \Omega(n_L^{-1})
 \end{align}
 The proof follows from combining eqs.
 \ref{eq:lower-2} and \ref{eq:lower-7}.
\end{proof}

\begin{proof}[\autoref{thm:lower_bound}]
 We wish to find a lower bound for
 the minimax risk given a constraint, $\mathcal{F}$. 
 In order to do so, we use the result that
 the minimax risk is lower bounded by the Bayes risk of
 any Bayes estimator  associated to a prior with
 support in $\mathcal{F}$ \citep{Wasserman2006,Esteves2017}.
 Since we consider the squared error loss,
 the Bayes risk of the Bayes estimator is
 $\E[\V[\theta|\XX_{U},\XX_{L},\Y_L]]$.
 Hence, if there exists two priors with support
 in $\mathcal{F}$ such that the first one satisfies
 $\E[\V[\theta|\XX_{U},\XX_{L},\Y_L]] \geq \Omega(n_{L}^{-1})$ and
 the second one satisfies
 $\E[\V[\theta|\XX_{U},\XX_{L},\Y_L]] \geq \Omega(n_{U}^{-1})$, then
 we can conclude that the minimax risk is
 lower bounded by $\Omega(\max(n_L^{-1},n_{U}^{-1}))$.
 Lemma \ref{lemma:lower-bound} can be used to determine
 these priors for the classes $\sF_{\sL_1,\epsilon}$
 and $\sF_{g, k, \epsilon}$.
 The proof for $\sF_{\sL_1,\epsilon}$ follows from
 taking $\epsilon^* = \epsilon$ in
 Lemma \ref{lemma:lower-bound}.
 Next, if $\sF_{g, k, \epsilon} \neq \emptyset$,
 then there exist $a$ and $b$ such that
 $\frac{\epsilon}{|g(a)-g(b)|} < 1$.
 Without loss of generality, let
 $a=0$ and $b=1$. The proof follows from
 taking $\epsilon^* = \frac{\epsilon}{|g(1)-g(0)|}$
 in Lemma \ref{lemma:lower-bound}.
\end{proof}

\begin{lemma}
 \label{lemma:theta-eq}
 For every function, $g$,
 under Assumption \ref{assump::cool1},
 \begin{align*}
  \theta := \P(Y=1|S=0)
  &= \frac{\E[g(\X)|S=0]-\E[g(\X)|Y=0,S=1]}
  {\E[g(\X)|Y=1,S=1]-\E[g(\X)|Y=0,S=1]}
 \end{align*}
\end{lemma}

\begin{proof}
 Let $f(\x)$ denote the density of $\X$.
 Note that
 \begin{align}
  \label{eq::ratio-1}
  g(\x)f(\x|S=0)
  &= \sum_{j=0}^1 g(\x)f(\x|Y=j,S=0)\P(Y=j|S=0)
  & \text{Law of total prob.} \nonumber \\
  \E[g(\X)|S=0]
  &=\sum_{j=0}^1 \E[g(\X)|Y=j,S=0]\P(Y=j|S=0)
  & \text{Integration over } \x \nonumber \\
  &=\sum_{j=0}^1 \E[g(\X)|Y=j,S=1]\P(Y=j|S=0)
  & \text{Assumption \ref{assump::cool1}}
 \end{align}
 Isolating $\P(Y=1|S=0)$ in 
 equation \ref{eq::ratio-1} yields
 \begin{align*}
  \theta := \P(Y=1|S=0)
  &= \frac{\E[g(\X)|S=0]-\E[g(\X)|Y=0,S=1]}
  {\E[g(\X)|Y=1,S=1]-\E[g(\X)|Y=0,S=1]}
 \end{align*}
\end{proof}

\begin{proof}[\autoref{thm:fisher-ratio}]
 Follows directly from
 the definition of $\widehat{\theta}_{UR}$ and
 $\widehat{\theta}_{R}$, and
 Lemma \ref{lemma:theta-eq}.
\end{proof}

\begin{lemma}
 \label{lemma:ratio}
 Let $Z_1$ and $Z_2$ be random variables such that
 $\E[Z_2] \neq 0$ and 
 $\frac{\E[Z_1]}{\E[Z_2]} \in [0,1]$.
 Define 
 $T = \max\left(0,\min\left(1,\frac{Z_1}{Z_2}\right)\right)$.
 For every random variable, $S$, and
 $\epsilon_1, \epsilon_2 \in (0,1)$.
 \begin{align*}
  \E\left[\left(T
  -\frac{\E[Z_1|S]}{\E[Z_2|S]}\right)^2\bigg|S\right]
  &\leq \frac{4(|\E[Z_1|S]|+\epsilon_1) 
  \max\left(\V[Z_1|S], \V[Z_2|S]\right)}
  {\min(1,(1-\epsilon_2)^4 \E[Z_2|S]^4)}
  +\epsilon_1^{-2}\V[Z_1|S] 
  +(\epsilon_2\E[Z_2|S])^{-2}\V[Z_2|S]
 \end{align*}
\end{lemma}

\begin{proof}
 It follows from Taylor's expansion of
 $\frac{Z_1}{Z_2}$ that 
 there exists $Z_{1,*}$ bounded between
 $\E[Z_1|S]$ and $Z_1$, and $Z_{2,*}$ between
 $\E[Z_2|S]$ and $Z_2$ such that
 \begin{align*} 
  \frac{Z_1}{Z_2}
  &= \frac{\E[Z_1|S]}{\E[Z_2|S]} 
  +\frac{1}{Z_{2,*}}(Z_1-\E[Z_1|S])
  -\frac{Z_{1,*}}{Z_{2,*}^2}(Z_2-\E[Z_2|S])
 \end{align*}
 Therefore, by letting
 $A = \{|Z_1-\E[Z_1|S]| \leq \epsilon_1, |Z_2-\E[Z_2|S]| \leq \epsilon_2 \E[Z_2|S]\}$, obtain
 \begin{align}
  \label{eq:lemma-ratio-1}
   &\E\left[\left(\frac{Z_1}{Z_2}
  -\frac{\E[Z_1|S]}{\E[Z_2|S]}\right)^2
  \bigg|A,S\right]\P(A|S) \nonumber \\
  =& \E\left[\left(\frac{1}{Z_{2,*}} (Z_1-\E[Z_1|S]) 
  -\frac{Z_{1,*}}{Z_{2,*}^2} (Z_2-\E[Z_2|S])\right)^2
  \bigg|A,S\right]\P(A|S) \nonumber \\
  \leq & 4\max\left(
  \E\left[\frac{1}{Z_{2,*}^2}(Z_1-\E[Z_1|S])^2\bigg|A,S\right],
  \E\left[\frac{Z_{1,*}^2}{Z_{2,*}^4}(Z_2-\E[Z_2|S])^2\bigg|A,S\right]\right)\P(A|S) \nonumber \\
  \leq & \frac{4(|\E[Z_1|S]|+\epsilon_1) \max\left(
  E\left[(Z_1-\E[Z_1|S])^2\big|A,S\right],
  E\left[(Z_2-\E[Z_2|S])^2\bigg|A,S\right]\right)\P(A|S)}
  {\min(1,(1-\epsilon_2)^4 \E[Z_2|S]^4)} \nonumber \\
  \leq &\frac{4(|\E[Z_1|S]|+\epsilon_1) \max\left(\V[Z_1|S],\V[Z_2|S]\right)}
  {\min(1,(1-\epsilon_2)^4 \E[Z_2|S]^4)}
 \end{align}
 Finally, obtain that 
 \begin{align*}
  E\left[\left(T-\frac{\E[Z_1|S]}{\E[Z_2|S]}\right)^2\bigg|S\right]
  &= \E\left[\E\left[\left(T-\frac{\E[Z_1|S]}{\E[Z_2|S]}\right)^2
  \bigg|\I_{A},S\right]\bigg|S\right] \nonumber \\
  &\leq \E\left[\left(\frac{Z_1}{Z_2}
  -\frac{\E[Z_1|S]}{\E[Z_2|S]}\right)^2
  \bigg|A,S\right]\P(A|S) + \P(A^c|S) 
  & T, \frac{\E[Z_1]}{\E[Z_2]} \in [0,1] \\
  &\leq  \frac{4(|\E[Z_1|S]|+\epsilon_1) 
  \max\left(\V[Z_1|S], \V[Z_2|S]\right)}
  {\min(1,(1-\epsilon_2)^4 \E[Z_2|S]^4)} + \P(A^c|S)
  & \text{eq. } \ref{eq:lemma-ratio-1}
 \end{align*}
 The result follows from applying 
 the union bound and
 Chebyshev's inequality to obtain
 \begin{align*}
  \P(A^c|S) &\leq \P(|Z_1-\E[Z_1|S]| > \epsilon_1|S)
  +\P(|Z_2-\E[Z_2|S]| > \epsilon_2\E[Z_2|S]|S) \\
  &\leq \epsilon_1^{-2}\V[Z_1|S] 
  + (\epsilon_2\E[Z_2|S])^{-2}\V[Z_2|S]
 \end{align*}
\end{proof}

\begin{proof}[\autoref{thm:mse-trimmed}]
 Define $Z_1 = \frac{\sum_{i \in A_0} g(\X_i)}{|A_0|}
 -\frac{\sum_{i\in A_{1,0}}g(\X_i)}{|A_{1,0}|}$ and also 
 $Z_2 = \frac{\sum_{i\in A_{1,1}}g(\X_i)}{|A_{1,1}|}
 -\frac{\sum_{i\in A_{1,0}}g(\X_i)}{|A_{1,0}|}$.
 Note that
 \begin{align*}
  \E\left[\frac{\sum_{i \in A_0} g(\X_i)}
  {|A_0|}\bigg|\S_1^n\right]
  &= \E[g(\X)|S=0] \\
  \E\left[\frac{\sum_{i\in A_{1,j}}g(\X_i)}
  {|A_{1,j}|}\bigg|\S_1^n\right]
  &= \E[g(\X)|Y=j,S=1]
 \end{align*}
 It follows from Lemma \ref{lemma:theta-eq} that
 $\theta = \frac{\E[Z_1|\S_1^n]}{\E[Z_2|\S_1^n]}$.
 With
 $T := \widehat{\theta}_{R} =  \max\left(0,\min\left(1,\frac{Z_1}{Z_2}\right)\right)$,
 obtain
 \begin{align*}
  \E\left[\left(\widehat{\theta}_{R}
  -\theta\right)^2\bigg|\S_1^n\right]
  &= \E\left[\left(T-\frac{\E[Z_1|\S_1^n]}{\E[Z_2|\S_1^n]}\right)^2\bigg|\S_1^n\right] \\
  &\leq \frac{4(|\E[Z_1|\S_1^n]|+\epsilon_1) 
  \max\left(\V[Z_1|\S_1^n], \V[Z_2|\S_1^n]\right)}
  {\min(1,(1-\epsilon_2)^4 \E[Z_2|\S_1^n]^4)} \\
  &+\epsilon_1^{-2}\V[Z_1|\S_1^n] 
  +(\epsilon_2\E[Z_2|\S_1^n])^{-2}\V[Z_2|\S_1^n]
  & \text{Lemma \ref{lemma:ratio}}
 \end{align*}
 The result follows from observing that,
 under $\sF_{g,K,\epsilon}$,
 $\E[Z_1|\S_1^n]$ and $\E[Z_2|\S_1^n]$ are
 bounded by constants,
 $\V[Z_1|\S_1^n] = \sO(\max(n_{L}^{-1},n_{U}^{-1}))$, and
 $\V[Z_2|\S_1^n] = \sO(n_{L}^{-1})$.
\end{proof}

\begin{proof}[\autoref{thm::EQMPriorShift}]
 The proof strategy is divided into two parts.
 The first part consists of proving a
 joint central limit theorem for
 the three sample averages that
 appear in the untrimmed ratio estimator.
 The second part uses this central limit theorem and
 the delta method to complete the proof for
 each case that is considered in the theorem.
 
 The main challenge appears when proving
 the central limit theorem for the sample averages that
 appear in the ratio estimator.
 This occurs since these averages are not marginally independent.
 However, they are independent conditional
 on the values of $\Y$ and $\S$.
 This conditional independence can
 be used to calculate the limiting behavior of
 the characteristic function of
 the standardized averages, which
 completes this part of the proof.
 
 We tidy the proof by using the following notation:
 $\mu_U := \E[g(\X_i)|S_i=0]$,
 $\sigma^2_U := \V[g(\X_i)|S_i=0]$,
 $Z_{U,n} := \frac{\sqrt{n_U}}{\sigma_U}
 \left(\frac{\sum_{i=1}^{n}g(\X_i)\I(S_i=0)}{n_U}-\mu_U\right)$,
 $Z_{j,n} := \frac{\sqrt{n_j}}{\sigma_j}
 \left(\frac{\sum_{i=1}^{n}g(\X_i)\I(S_i=0,Y_i=j)}{n_{j}}-\mu_j\right)$,
 $F_i = \I(S_i=1)(Y_i + 1)$, $A_U = \{F_1=0\}$, $A_0 = \{F_1=1\}$, 
 and $A_1 = \{F_1=2\}$. Note that
 \begin{align}
  \label{eq:clt-1}
  \limn \phi_{Z_{U,n},Z_{0,n},Z_{1,n}}(t_U,t_0,t_1)
  &= \limn \E\left[\E\left[\exp\left( 
  \sum_{j \in \{U,0,1\}}{it_j Z_{j,n}}\right)
  \bigg|F_1,\ldots,F_n\right]\right] 
  \nonumber \\
  &= \limn \E\left[
  \prod_{j \in \{U,0,1\}} 
  \E\left[\exp\left(it_j Z_{j,n} \right)
  \bigg|F_1,\ldots,F_n\right]\right] 
  \nonumber \\
  &= \limn \E\left[
  \prod_{j \in \{U,0,1\}} 
  \left(\phi_{\frac{g(\X_1)-\mu_j}{\sigma_j}\big|A_j}
  (t_j n_j^{-0.5})\right)^{n_j} \right]
 \end{align}
 It follows from the Central Limit Theorem for
 i.i.d. random variables that,
 for every $j \in \{U,0,1\}$,
 $\phi^{n_j}_{\frac{g(\X_1)-\mu_j}{\sigma_j}\big|A_j}
 (t_j n_j^{-0.5}) \rightarrow \exp(-0.5t_j^2)$
 as $n_j \rightarrow \infty$.
 Since $n_j \convas \infty$, 
 conclude from eq. \ref{eq:clt-1} and
 the dominated convergence theorem that
 \begin{align*}
   \limn \phi_{Z_{U,n},Z_{0,n},Z_{1,n}}(t_U,t_0,t_1)
  &= \prod_{j \in \{U,0,1\}} \exp(-0.5 t_j^2)
 \end{align*}
 and, using $\II$ as the identity matrix, obtain
 \begin{align}
  \label{eq:clt-2}
  \left(Z_{U,n}, Z_{0,n}, Z_{1,n}\right) 
  &\convd N(0,\II)
 \end{align}

 Assume that $p_L \neq 0$. In this case,
 since $\frac{n_L}{n} \convp p_L$,
 it follows from eq. \ref{eq:clt-2} that
 \begin{align*}
  \sqrt{n}\left(
  \frac{\sum_{i=1}^{n}g(\X_i)\I(S_i=0)}{n_U}-\mu_U,
  \frac{\sum_{i=1}^{n}g(\X_i)\I(S_i=0,Y_i=0)}{n_{0}}-\mu_{0},
  \frac{\sum_{i=1}^{n}g(\X_i)\I(S_i=0,Y_i=1)}{n_{1}}-\mu_{1}
  \right)
 \end{align*}
 converges in distribution to
 $N\left(0,diag\left(\frac{\sigma_U^2}{1-p_L},
 \frac{\sigma_0^2}{p_L p_{0|L}},
 \frac{\sigma_1^2}{p_L p_{1|L}}\right)\right)$.
 Since $\theta = \frac{\mu_U-\mu_{0}}{\mu_{1}-\mu_{0}}$ 
 (Lemma \ref{lemma:theta-eq}) and
 $\widehat{\theta}_{UR} = 
 \frac{\frac{\sum_{i=1}^{n}g(\X_i)\I(S_i=0)}{n_U}
 -\frac{\sum_{i=1}^{n}g(\X_i)\I(S_i=0,Y_i=0)}{n_{0}}}
 {\frac{\sum_{i=1}^{n}g(\X_i)\I(S_i=0,Y_i=1)}{n_{1}}
 -\frac{\sum_{i=1}^{n}g(\X_i)\I(S_i=0,Y_i=0)}{n_{0}}}$, 
 it follows from the delta method
 \citep{Casella2002} that
 \begin{align*}
  \sqrt{n}(\widehat{\theta}_{UR}-\theta)
  &\convd N\left(0,
  \frac{\sigma_U^2 (1-p_L)^{-1} }{(\mu_{1}-\mu_{0})^2}
  +\frac{(\mu_U-\mu_{1})^2 \sigma_0^2 (p_L p_{0|L})^{-1}}
  {(\mu_{1}-\mu_{0})^4}
  +\frac{(\mu_U-\mu_{0})^2 \sigma_1^2 (p_L p_{1|L})^{-1}}
  {(\mu_{1}-\mu_{0})^4} \right)
 \end{align*}
 Since $\mu_U = (1-\theta)\mu_{0} + \theta \mu_{1}$ and
 $\sigma_U^2 = (1-\theta)\sigma_0^2 + \theta \sigma_1^2 + (\mu_1-\mu_0)^2 \theta(1-\theta)$
 obtain that
 \begin{align*}
  \sqrt{n}(\widehat{\theta}_{UR}-\theta)
  &\convd N\left(0,
  \frac{\frac{(1-\theta)\sigma_0^2 + \theta \sigma_1^2 + (\mu_1-\mu_0)^2 \theta(1-\theta)}{1-p_L}
  +\frac{(1-\theta)^2 \sigma_0^2}{p_L p_{0|L}}
  +\frac{\theta^2 \sigma_1^2}{p_L p_{1|L}}}
  {(\mu_1-\mu_0)^2} \right)
 \end{align*}
 
 Next, assume that $p_{L}=0$. 
 Obtain that $\sqrt{h(n)}(Z_{U,n}-\mu_U) \convp 0$ and
 \begin{align*}
  \sqrt{h(n)}\left(
  \frac{\sqrt{p_{0|L}}}{\sigma_0}
  (Z_{0,n}-\mu_{0}),
  \frac{\sqrt{p_{1|L}}}{\sigma_1}
  (Z_{1,n}-\mu_{1})
  \right)
  & \convd N(0,\II)
 \end{align*}
 It follows from the delta method and
 Slutsky's theorem that
 \begin{align*}
  \sqrt{h(n)}(\widehat{\theta}_{UR}-\theta)
  &\convd N\left(0,
  \frac{\frac{(1-\theta)^2 \sigma_0^2}{p_{0|L}}
  +\frac{\theta^2 \sigma_1^2}{p_{1|L}}}
  {(\mu_1-\mu_0)^2} \right)
 \end{align*}
 
 The same convergence results hold for
 $\widehat{\theta}_{R}$ since
 the derivative of the trimming function is
 $1$ around $\theta$.
\end{proof}

\begin{proof}[\autoref{thm:rkhs}]
 It follows from the Representer Theorem \citep{Wahba1990} that,
 for every $g \in \mathcal{H}_K$,
 $g(\x) = \sum_{k \in A_1} w_k K(\x,\x_k)$.
 Using this fact, for every $i \in \{0,1\}$,
 \begin{align*}
  \widehat{\mu}_i
  &= \frac{\sum_{j \in A_{1,i}}g(\x_j)}{n_{i}}
  = \frac{\sum_{k \in A_1} w_k 
  \sum_{j \in A_{1,i}} K(\x_j, \x_k)}{n_i}
  =\vec{w}^t m_i \\
  \widehat{\sigma}_i^2
  &= \frac{\sum_{j \in A_{1,i}}
  (g(\x_j)-\widehat{\mu}_i)^2}{n_i}
  = \vec{w}^t \widehat{\Sigma}_i \vec{w}
 \end{align*}
 Therefore, for every $g \in \mathcal{H}_K$,
%  \begin{align*}
%   \widehat{\mbox{MSE}}(g)
%   +\lambda ||g||_{\mathcal{H}_K}
%   &=\frac{\vec{w}^t N \vec{w}}
%   {\vec{w}^t M \vec{w}}
%   +\lambda \vec{w}^t \mathbb{K} \vec{w}.
%  \end{align*}
 \begin{align*}
  \widehat{\mbox{MSE}}(g)
  &=\frac{\vec{w}^t N \vec{w}}
  {\vec{w}^t M \vec{w}}.
 \end{align*}
\end{proof}

\begin{proof}[Proposition \ref{lemma:priorShiftConsequence}]
 \begin{align*}
  F_{g(\X)|S=0} 
  &=\theta  F_{g(\X)|S=0,Y=1} + (1-\theta) F_{g(\X)|S=0,Y=0} \\
  &=\theta F_{g(\X)|S=1,Y=1} + (1-\theta )F_{g(\X)|S=1,Y=0}
  & \text{\ref{assump::cool1}}
 \end{align*}
 Thus, there exists 
 $0\leq p\leq  1\mbox{ such that } 
 F_{g(\X)|S=0} = pF_{g(\X)|S=1,Y=1} + (1-p)F_{g(\X)|S=1,Y=0}$.
\end{proof}

\begin{proof}[\autoref{thm:combined}]
\begin{align*}
 MSE[\widehat{\theta}_C]
 &=\E\left[\left((w\widehat{\theta}_R+(1-w)\widehat{\theta}_L)-\theta\right)^2\right]=
 \E\left[\left((w(\widehat{\theta}_R-\theta)+(1-w)(\widehat{\theta}_L-\theta)\right)^2\right]
  \\
 &=w^2 MSE[\widehat{\theta}_R]+(1-w)^2MSE[\widehat{\theta}_L]+2w(1-w)\E[(\widehat{\theta}_R-\theta)(\widehat{\theta}_L-\theta)] 
 & \widehat{\theta}_R \text{ indep. }
 \widehat{\theta}_L, \E[\widehat{\theta}_L]=\theta \\ 
 &=w^2 MSE[\widehat{\theta}_R]+(1-w)^2MSE[\widehat{\theta}_L]
\end{align*}
It follows that  $MSE[\widehat{\theta}_C]$ is 
minimized by taking
$w = MSE[\widehat{\theta}_L]\times (MSE[\widehat{\theta}_L] + MSE[\widehat{\theta}_R])^{-1}$.
\end{proof}

\begin{lemma}
 \label{lemma:reg}
 For every function, $g$,
 under Assumptions \ref{assump::cool1} and
 \ref{assump::condCovIndep}
 \begin{align*}
  \theta(\z) &= 
  \frac{\E[g(\X)|S=0,\z]-\E[g(\X)|Y=0,S=1]}
  {\E[g(\X)|Y=1,S=1]-\E[g(\X)|Y=0,S=1]}
 \end{align*}
\end{lemma}

\begin{proof}
 For every $\z \in \Re^{d_z}$,
Let $f(\x)$ denote the density of $\X$.
 Note that
 \begin{align}
  \label{eq::ratio-reg}
  g(\x)f(\x|S=0,\z)
  &= \sum_{j=0}^1 g(\x)f(\x|Y=j,S=0,\z)\P(Y=j|S=0,\z)
  & \text{Law of total prob.} \nonumber \\
  \E[g(\X)|S=0,\z]
  &= \sum_{j=0}^1 \E[g(\X)|Y=j,S=0,\z]\P(Y=j|S=0,\z)
  & \text{Integration over } \x \nonumber \\
  &= \sum_{j=0}^1 \E[g(\X)|Y=j,S=0]\P(Y=j|S=0,\z)
  & \text{Assumption \ref{assump::condCovIndep}}
  \nonumber \\
  &= \sum_{j=0}^1 \E[g(\X)|Y=j,S=1]\P(Y=j|S=0,\z)
  & \text{Assumption \ref{assump::cool1}}
 \end{align}
 Isolating $\P(Y=1|S=0,\z)$ in 
 equation \ref{eq::ratio-reg} yields
 \begin{align*}
  \theta(\z) := \P(Y=1|S=0,\z)
  &= \frac{\E[g(\X)|S=0,\z]-\E[g(\X)|Y=0,S=1]}
  {\E[g(\X)|Y=1,S=1]-\E[g(\X)|Y=0,S=1]}
 \end{align*}
\end{proof}

\begin{proof}[\autoref{thm:rr-mse}]
The main difference between this proof and
the one of \autoref{thm:mse-trimmed} is that
$\widehat{\E}[g(\X)|S=0,\z]$ is usually biased for
$\E[g(\X)|S=0,\z]$. The proof strategy consists of
isolating this bias term from the squared error and
then replicating steps which are similar to
the ones in \autoref{thm:mse-trimmed}.

In order to present the proof in a compact form,
some special notation is used. Specifically,
$h_0(\z) = \E[g(\X)|S=0,\z]$,
$\hat{h}_0(\z) = \widehat{\E}[g(\X)|S=0,\z]$,
$h_{1,i} = \E[g(\X)|S=1, Y=i]$, and
$\hat{h}_{1,i} = \widehat{\E}[g(\X)|S=1, Y=i]$.
Using this notation and
 Definition \ref{defn:reg_ratio}, note that
 $\hat{\theta}_{RR}(\Z) = \max\left(0,\min \left(1, \frac{\hat{h}_0(\z)-\hat{h}_{1,0}}{\hat{h}_{1,1}-\hat{h}_{1,0}} \right)\right)$. Therefore,
 
 \begin{align*}
  & \E\left[\left(\widehat{\theta}_{RR}(\Z)
  - \theta(\Z) \right)^2\bigg|S_1^n\right] \\
  = & \E\left[\left(
  \widehat{\theta}_{RR}(\Z) 
  - \frac{h_0(\Z) - h_{1,0}}{h_{1,1} - h_{1,0}}
  \right)^2\bigg|S_1^n\right] 
  & \text{Lemma \ref{lemma:reg}} \\
  \leq & \sO\left(\E\left[\left(
  \widehat{\theta}_{RR}(\Z) 
  - \frac{\E[\hat{h}_0(\Z)|S_1^n] - h_{1,0}}{h_{1,1} - h_{1,0}}\right)^2
  + \left(\frac{\E[\hat{h}_0(\Z)|S_1^n] - h_{1,0}}{h_{1,1} - h_{1,0}}
  - \frac{h_0(\Z) - h_{1,0}}{h_{1,1} - h_{1,0}}
  \right)^2\bigg|S_1^n\right]\right) \\
  \leq & \sO\left(
  \V[\hat{h}_{1,0}|S_1^n] + \V[\hat{h}_{1,1}|S_1^n]
  + \V[\hat{h}_0(\Z)|S_1^n]
  + \left(\E[\hat{h}_0(\Z)|S_1^n] - h_0(\Z)\right)^2\right)
  & \text{Lemma \ref{lemma:ratio}} \\
  = & \sO\left(\max\left( \V[\hat{h}_{1,0}|S_1^n], \V[\hat{h}_{1,1}|S_1^n],
  \E\left[(\hat{h}_0(\Z)-h_0(\z))^2|S_1^n\right]\right)\right) \\
  = & \sO\left(\max\left(n_L^{-1},
  \E\left[(\hat{h}_0(\Z)-h_0(\z))^2|S_1^n\right]\right)\right)
 \end{align*}
 The last equality follows from observing that
 $\hat{h}_{1,0}$ and $\hat{h}_{1,1}$ are sample averages.
\end{proof}

\newpage

  %& \text{Lemma } \\
 % \leq & \sO\left(\max\left(\V[\hat{h}_0(\Z)|S_1^n],
 % \V[\hat{h}_{1,0}|S_1^n],
 % (\E[\hat{h}_0(\Z)|S_1^n]-h_0(\Z))^2
 % \right)\right) 
 % & \text{} \\
 % \leq & \sO(\max(
 % \E[(\widehat{\E}[g(\X)|S=0,\Z]-\E[g(\X)|S=0,\Z])^2|S_1^n], n_{L}^{-1}))
 
 \section{Additional Figures}
 \label{ap:figures}
 
 These are additional figures to the experiments of Sections \ref{sec:expRatio} and \ref{sec:combined}.
 
 \begin{figure}[h!]
 \centering
 \makebox[\textwidth]{
  \includegraphics[angle=90 ,origin=c,page=1,scale=0.51]
  {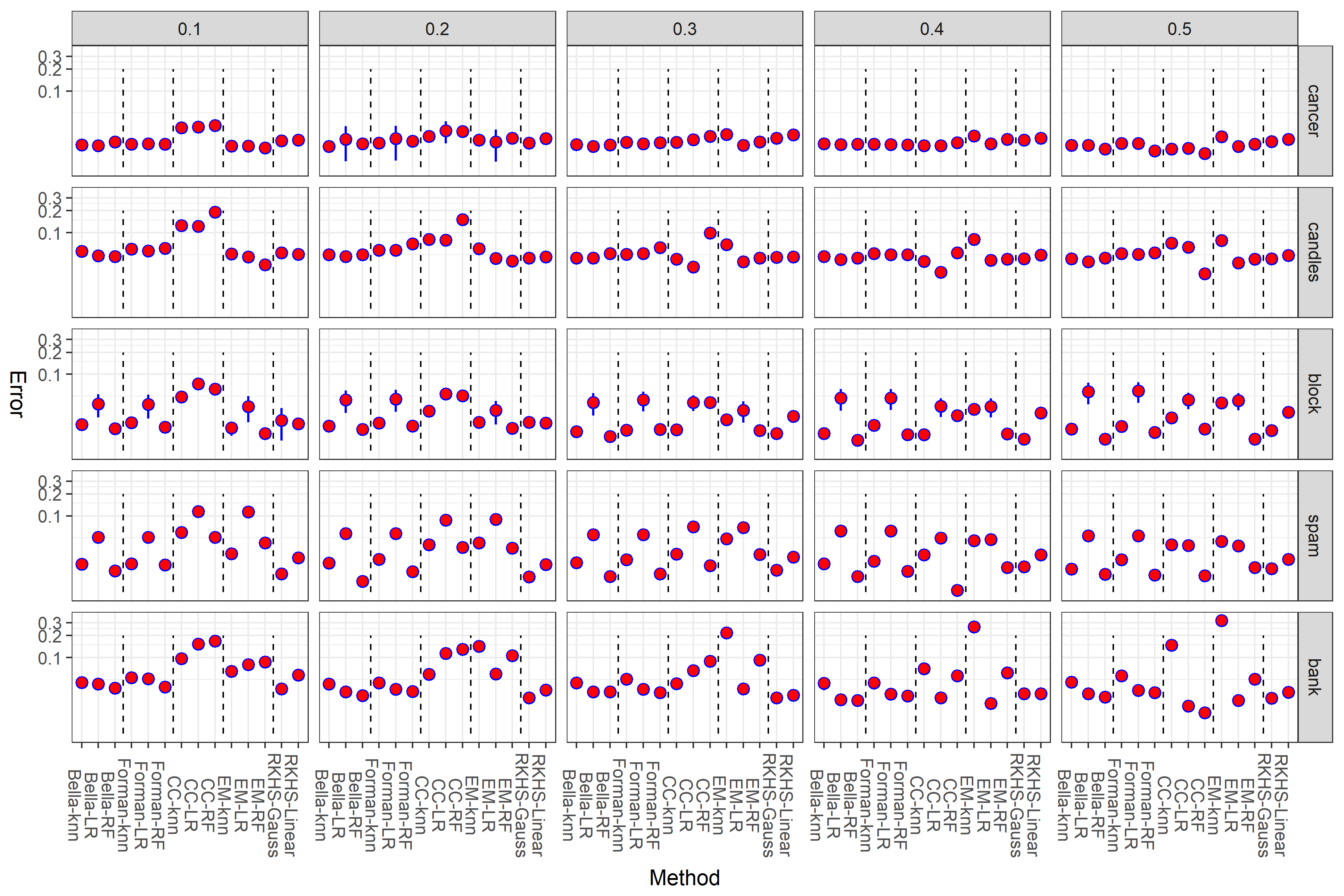}
 }
 \caption{Root mean square deviation of
 each method by setting in
 logarithmic scale (including K-NN estimator).}
 \label{fig::error_all}
\end{figure}

\begin{figure}
 \centering
 \makebox[\textwidth]{
  \includegraphics[page=1,scale=0.4]
  {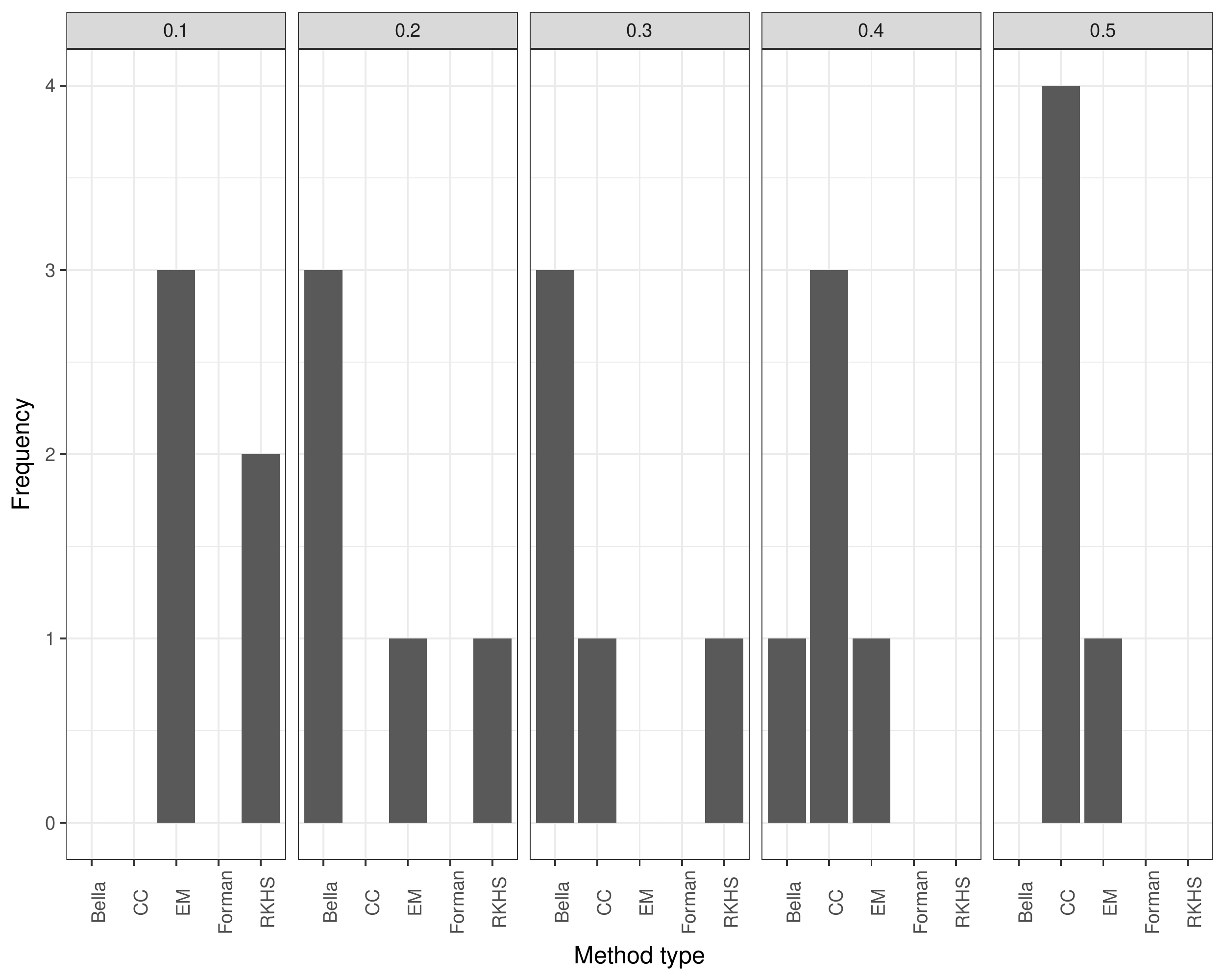}
 }
 \caption{Number of times in which each specific method presents smaller MSE  by $\theta$ values.}
 \label{fig::count_theta}
\end{figure}

\begin{figure}
 \centering
 \makebox[\textwidth]{
  \includegraphics[page=1,scale=0.4]
  {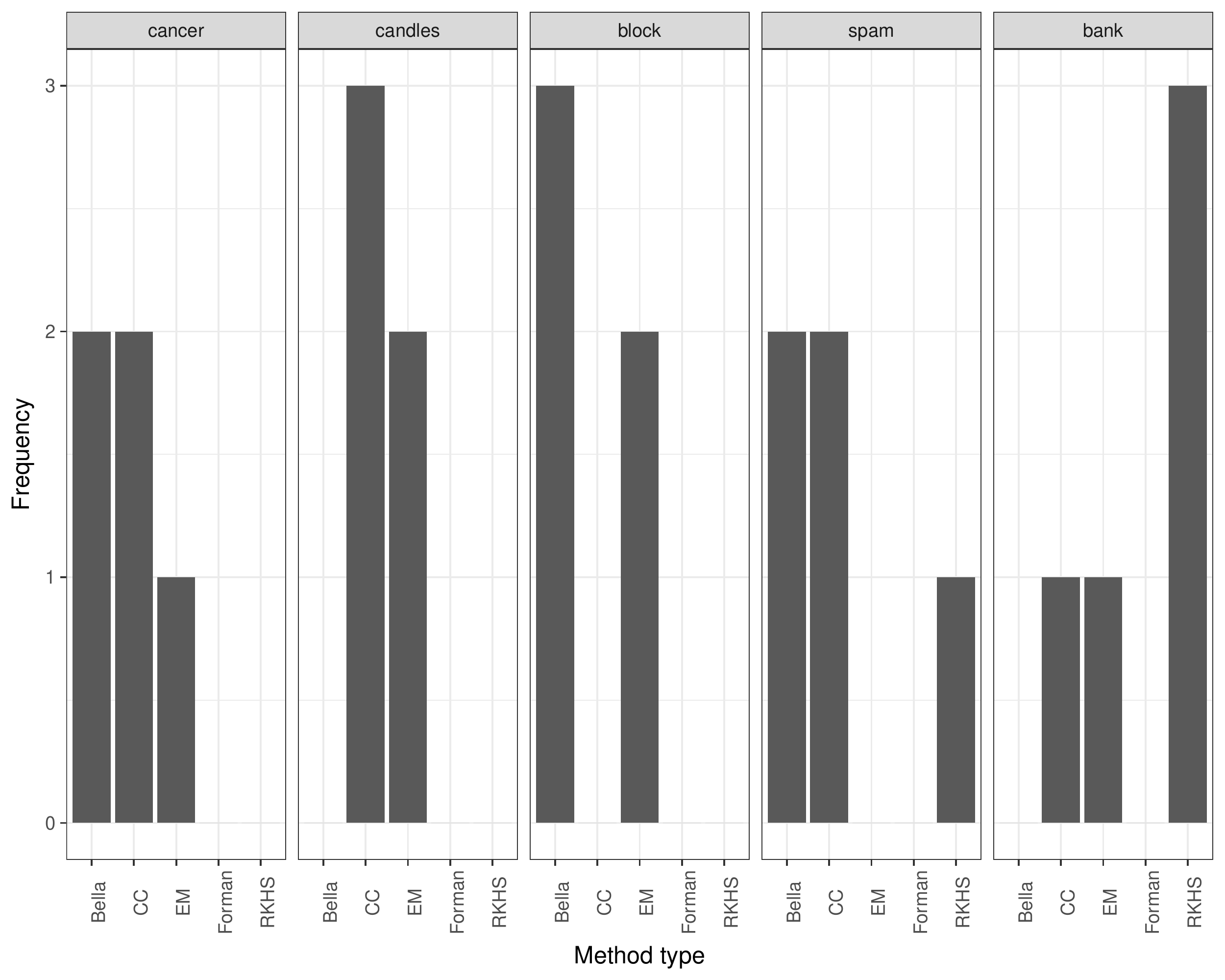}
 }
 \caption{Number of times in which each specific method presents smaller MSE by data set.}
 \label{fig::count_dataset}
\end{figure}

\begin{figure}
 \centering
 \makebox[\textwidth]{
  \includegraphics[page=1,scale=0.51]{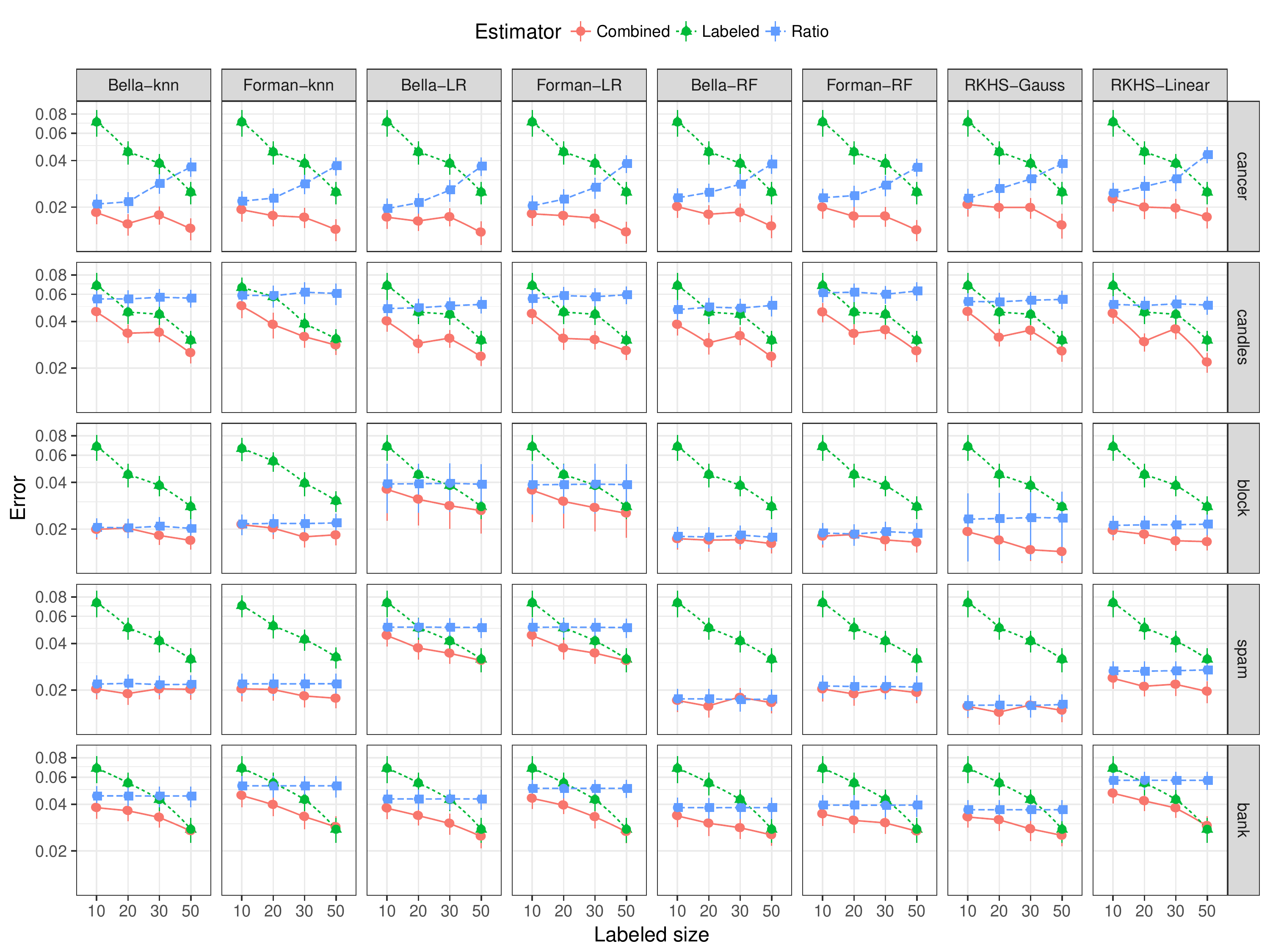}}
 \caption{Root mean square deviation in
 logarithmic scale for
 each data set, method and estimator by
 size of the labeled sample and using $\theta=0.1$.}
 \label{fig::comb2}
\end{figure}

\begin{figure}
 \centering
 \makebox[\textwidth]{
  \includegraphics[page=2,scale=0.51]{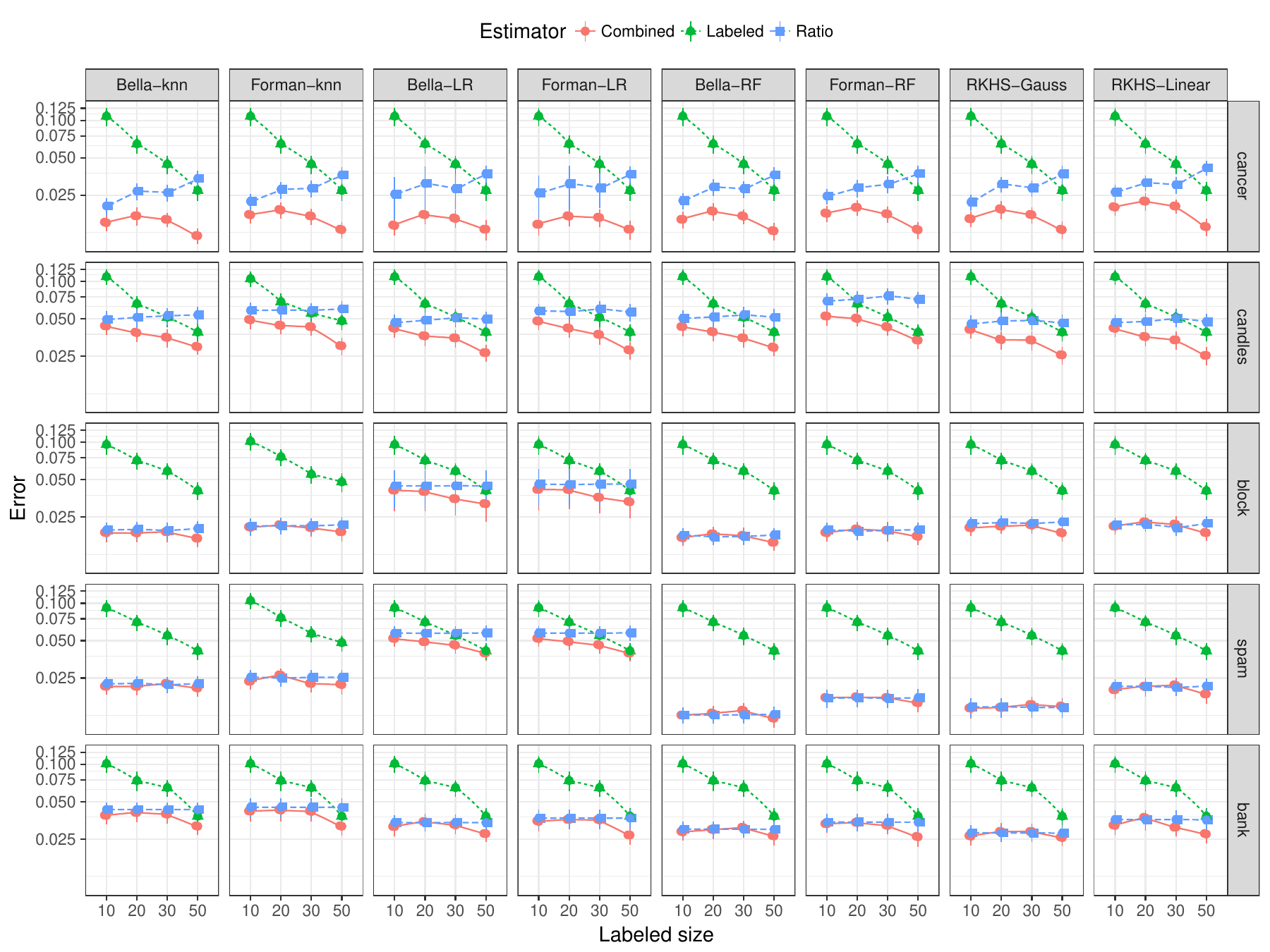}}
 \caption{Root mean square deviation in
 logarithmic scale for
 each data set, method and estimator by
 size of the labeled sample and using $\theta=0.2$.}
 \label{fig::comb3}
\end{figure}

\begin{figure}
 \centering
 \makebox[\textwidth]{
  \includegraphics[page=4,scale=0.51]{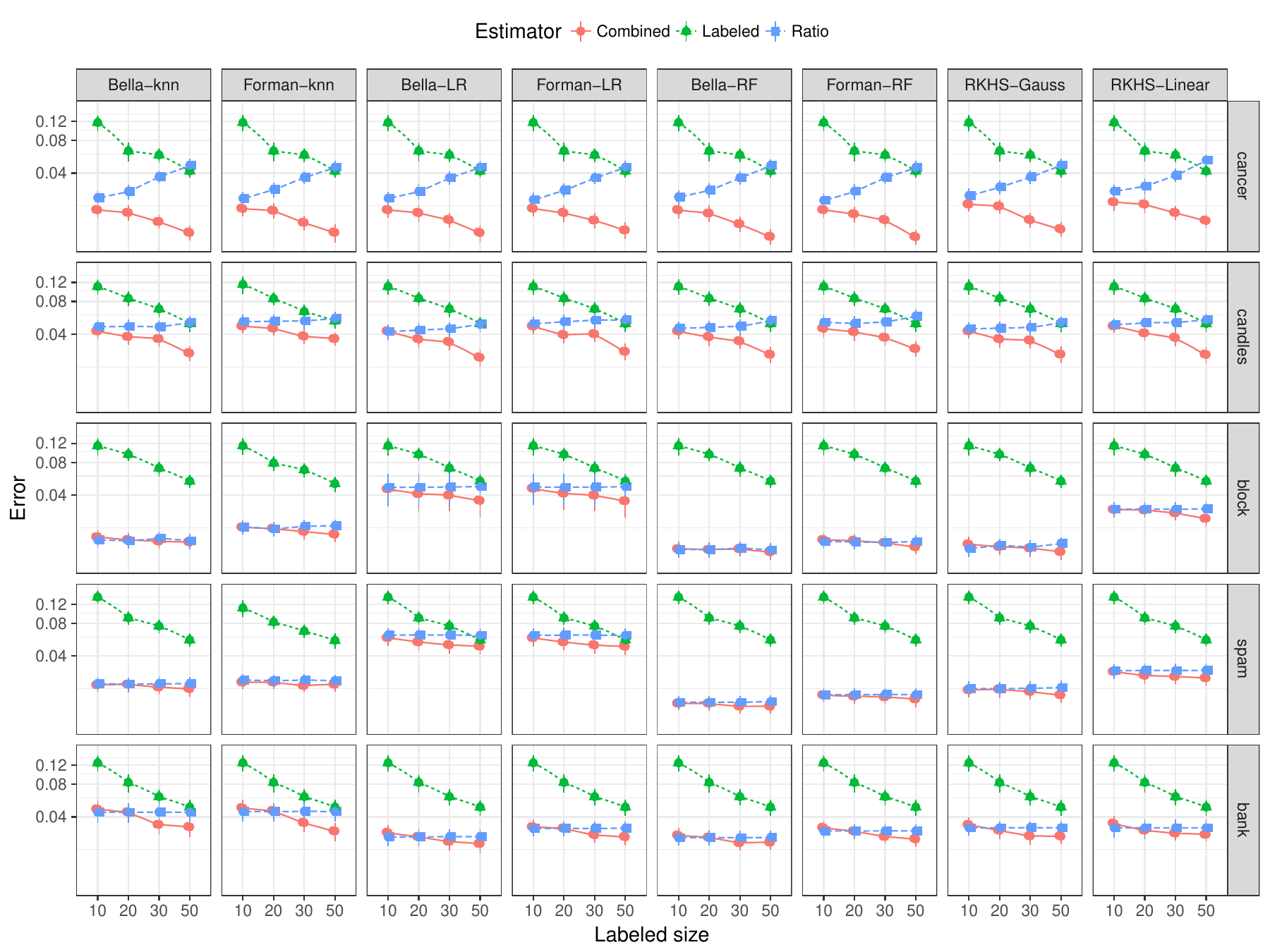}}
 \caption{Root mean square deviation in
 logarithmic scale for
 each data set, method and estimator by
 size of the labeled sample and using $\theta=0.4$.}
 \label{fig::comb4}
\end{figure}

\begin{figure}
 \centering
 \makebox[\textwidth]{
  \includegraphics[page=5,scale=0.51]{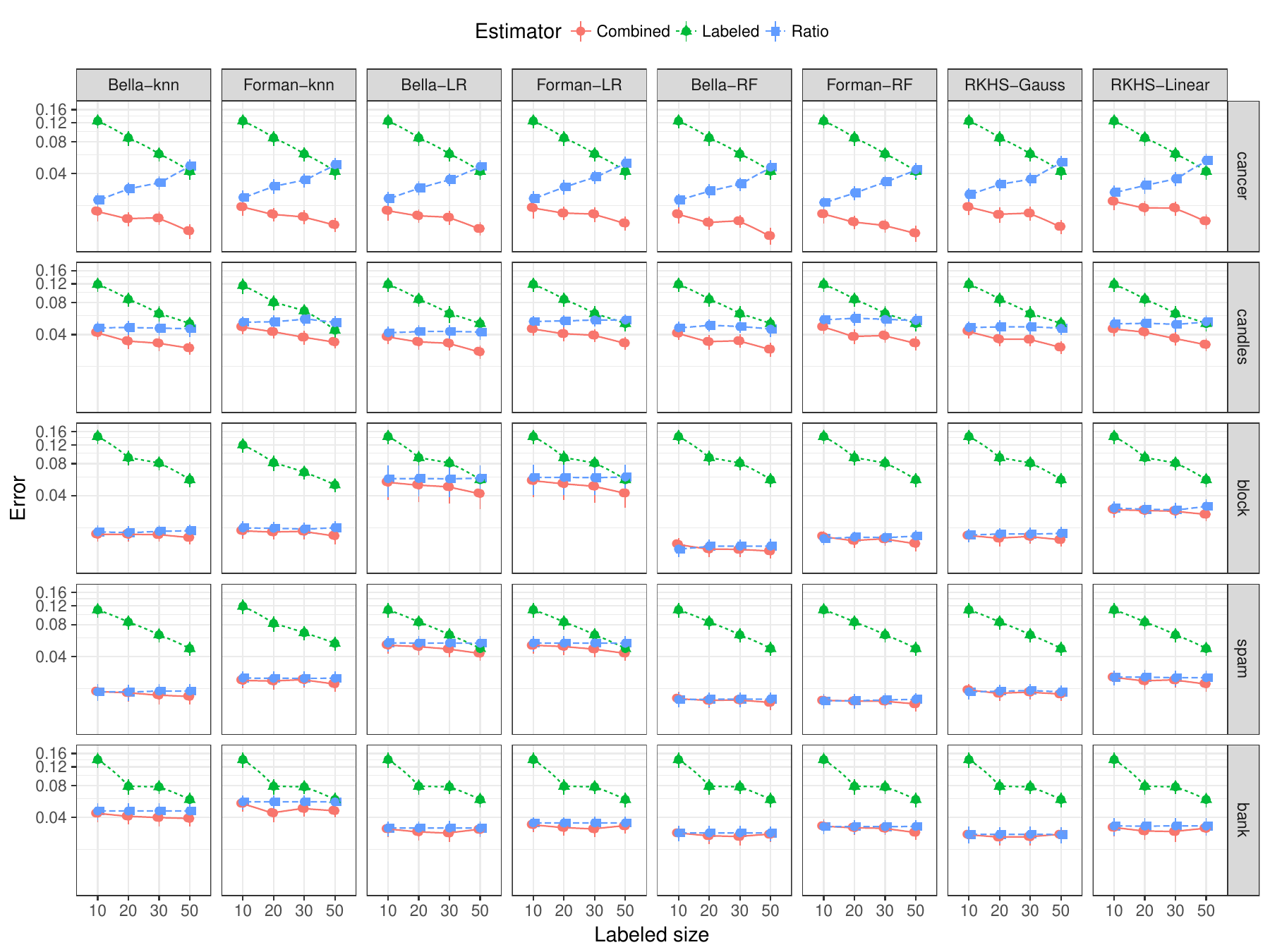}}
 \caption{Root mean square deviation in
 logarithmic scale for
 each data set, method and estimator by
 size of the labeled sample and using $\theta=0.5$.}
 \label{fig::comb5}
\end{figure}
\bibliography{18-456}

\end{document}